\documentclass{ieeeaccess}
\usepackage{cite}
\usepackage{amsmath,amssymb,amsfonts}
\usepackage{algorithmic}
\usepackage{graphicx}
\usepackage{textcomp}
\usepackage{amsthm}
\usepackage{bm}
\usepackage{booktabs}
\usepackage{graphicx}
\usepackage{hyperref}  
\usepackage[caption=false,font=footnotesize]{subfig} 

\definecolor{revcolor}{RGB}{255,204,0}
\newcommand{\rev}[1]{\textcolor{revcolor}{#1}}
\renewcommand{\rev}[1]{#1}
\newtheorem{theorem}{Theorem}
\newtheorem{lemma}[theorem]{Lemma}

\newtheorem{proposition}[theorem]{Proposition}

\usepackage{bm}
\makeatletter
\AtBeginDocument{\DeclareMathVersion{bold}
\SetSymbolFont{operators}{bold}{T1}{times}{b}{n}
\SetSymbolFont{NewLetters}{bold}{T1}{times}{b}{it}
\SetMathAlphabet{\mathrm}{bold}{T1}{times}{b}{n}
\SetMathAlphabet{\mathit}{bold}{T1}{times}{b}{it}
\SetMathAlphabet{\mathbf}{bold}{T1}{times}{b}{n}
\SetMathAlphabet{\mathtt}{bold}{OT1}{pcr}{b}{n}
\SetSymbolFont{symbols}{bold}{OMS}{cmsy}{b}{n}
\renewcommand\boldmath{\@nomath\boldmath\mathversion{bold}}}
\makeatother

\def\BibTeX{{\rm B\kern-.05em{\sc i\kern-.025em b}\kern-.08em
    T\kern-.1667em\lower.7ex\hbox{E}\kern-.125emX}}

\begin{document}
\doi{10.1109/ACCESS.2026.3669415}

\title{The density of cross-persistence diagrams and its applications}
\author{\uppercase{Alexander Mironenko}\authorrefmark{1},
\uppercase{Evgeny Burnaev}\authorrefmark{1,2} and \uppercase{Serguei Barannikov}\authorrefmark{1,3}}

\address[1]{Skolkovo Institute of Science and Technology, Moscow 121205, Russia}
\address[2]{AIRI, Moscow 123317, Russia}
\address[3]{CNRS, IMJ, Paris University, Paris 75205, France}


\markboth
{Alexander Mironenko \headeretal: Density of Cross-Persistence Diagrams and Its Applications}
{Alexander Mironenko \headeretal: Density of Cross-Persistence Diagrams and Its Applications}


\begin{abstract}
Topological Data Analysis (TDA) provides powerful tools to explore the shape and structure of data through topological features such as clusters, loops, and voids. Persistence diagrams are a cornerstone of TDA, capturing the evolution of these features across scales. While effective for analyzing individual manifolds, persistence diagrams do not account for interactions between pairs of them. Cross-persistence diagrams (cross-barcodes), introduced recently, address this limitation by characterizing relationships between topological features of two point clouds.
In this work, we \rev{present} the first systematic study of the density of cross-persistence diagrams. We prove its existence, establish theoretical foundations for its statistical use, and design the first machine learning framework for predicting cross-persistence density directly from point cloud coordinates and distance matrices. Our statistical approach enables the distinction of point clouds sampled from different manifolds by leveraging the linear characteristics of cross-persistence diagrams. Interestingly, we \rev{find} that introducing noise \rev{can enhance} our ability to distinguish point clouds, uncovering its novel utility in TDA applications.
We demonstrate the effectiveness of our methods through experiments on diverse datasets, where our approach consistently outperforms existing techniques in density prediction and achieves superior results in point cloud distinction tasks. Our findings contribute to \rev{a} broader understanding of cross-persistence diagrams and open new avenues for their application in data analysis, including potential insights into time-series domain tasks and the geometry of AI-generated texts. Our code is publicly available at \url{https://github.com/Verdangeta/TDA_experiments}.
\end{abstract}


\titlepgskip=-21pt

\maketitle

\section{Introduction}
\label{sec:introduction}
\PARstart{T}{he} study of topological structures in data has gained significant attention in recent years, particularly in the context of high-dimensional data analysis. Topological Data Analysis (TDA) provides a framework for extracting and comparing topological features across domains, supporting robust methodologies in shape recognition, anomaly detection, and generative model evaluation. A notable advancement in this area is the introduction of cross-persistence diagrams (cross-barcodes), a tool designed to compare distributions by capturing multiscale topological discrepancies between manifolds \cite{barannikov2021manifold}.

This is particularly relevant for generative modeling, where evaluating the quality of synthetic data remains a fundamental challenge. Traditional metrics often fail to capture structural differences between real and generated samples, especially in high-dimensional settings. The cross-persistence framework addresses this by providing a topological divergence measure between two distributions. Building on this idea, the Manifold Topology Divergence (MTop-Divergence) has been proposed as a practical tool for assessing generative models in domains such as image synthesis, 3D shape generation, and time-series modeling.

\rev{However, existing cross-persistence-based methods operate at the level of individual diagrams or summary statistics and do not provide a principled notion of a probability density over cross-persistence diagrams, nor learning-based tools for its estimation. At the same time, despite their utility, cross-persistence diagrams suffer from high computational complexity, as they require interactions between two sets of topological features, substantially increasing the cost of computation.}


In this paper, we present the first theoretical and algorithmic framework for cross-persistence density. We not only prove its existence and statistical utility, but also introduce machine learning methods, including a novel neural model, Cross-RipsNet, to predict and exploit this density for practical data analysis tasks. Our framework supports both theoretical investigations and applied studies, ranging from refined manifold comparison and point cloud classification to domain-specific applications such as time-series analysis and AI-generated text identification.


We summarize our contributions as follows:
\begin{itemize}
\item We propose a methodology with rigorous theoretical foundations for estimating and utilizing the density function of cross-persistence diagrams, enabling refined comparison of data manifolds.
\item We introduce a statistical approach for distinguishing manifolds based on the linear characteristics of cross-persistence diagrams derived from samples on different manifolds, revealing the novel utility of noise in improving separability.
\item We show that introducing noise can enhance the effectiveness of our statistical method in separating point clouds.
\item We introduce Cross-RipsNet, the first neural architecture for learning cross-persistence densities from point cloud coordinates and distance matrices.
\item We present extensive experiments across multiple data modalities, demonstrating superior performance in density prediction and in distinguishing point clouds originating from different manifolds.
\item We illustrate practical impact through case studies in time-series classification (e.g., gravitational wave detection) and AI-generated text identification.
\end{itemize}

\section{Related Works}
\label{sec:related}
Early foundational work 
\cite{barannikov:morsecomplex, edelsbrunner2002topological, zomorodian2005computing} introduced persistent homology as a method \rev{for encoding} the topological structure of data given in the form of scalar fields or point clouds through persistence diagrams and barcodes. 
Subsequent research expanded their applicability through vectorization techniques,
including persistence landscapes \cite{bubenik2015statistical}, persistence images \cite{adams2016persistenceimages}, and other representations.

In recent years, persistence diagrams have been extensively studied, with a growing emphasis on cross-persistence, a paradigm that analyzes interactions between two manifolds, two weighted graphs or two scalar functions. In \cite{barannikov2021manifold}, the cross-barcode and Manifold Topology Divergence (MTD), a tool for assessing the performance of deep generative models such as GANs, were introduced. Subsequent development in \cite{barannikov2022representation} proposed a framework for comparing latent representations of data, allowing to gain insights into neural network representations in CV and NLP domains.

The proliferation of cross-persistence methods has spurred applications beyond traditional domains. Recent works \cite{tulchinskii2023intrinsicdimensionestimationrobust, Tulchinskii_2023} demonstrate the versatility of these tools, addressing challenges in AI-generated text detection and speech classification. By leveraging the intrinsic topological structure of data, their approach achieves SOTA results in many domains, emphasizing the broader potential of cross-persistence in modern machine learning tasks ~\cite{rieck2020uncoveringtopologytimevaryingfmri, kushnareva2021artificial}.

\section{Background}
\label{sec:methods}

In this section, we review the fundamental concepts of persistent homology, the construction of Rips persistence diagrams, and introduce the Cross-Barcode.

\subsection{Persistence Barcodes}

Let \( X \) be a topological space and \( f : X \to \mathbb{R} \) be a continuous function. The filtration induced by \( f \) is a nested sequence of sublevel sets defined as:
\begin{equation}
X_{\alpha} = \{ x \in X : f(x) \leq \alpha \}, \quad \alpha \in \mathbb{R}.
\end{equation}
As $\alpha$ increases, these sublevel sets define the filtration on chain complexes, inducing the evolution of the topological features , such as connected components, loops, and voids. Persistence barcode tracks this evolution by focusing on the birth and death of fundamental topological features across different scales.

Each basic topological feature, be it a connected component, a loop, or a void, is associated with a persistence interval, defined by its birth time $\alpha_b$ (when the feature first appears) and its death time $\alpha_d$ (when the feature disappears). The pair ($\alpha_b$, $\alpha_d$) encapsulates the "lifetime" of the feature as the parameter $\alpha$ increases.

\subsection{ Persistence Diagrams}

Assume we are dealing with a finite point cloud \( X = \{ p_1, p_2, \dots, p_n \} \) in metric space $(M, \rho)$ and a real number $\alpha \geqslant 0$. The Vietoris-Rips simplicial complex $Rips_{\alpha}(X)$ consists of the set of all simplices with vertices from $X$ such that the distances between the vertices do not exceed $\alpha$.

The Vietoris-Rips filtration $K(X) = (K(X,\alpha))_{\alpha}$ is the sequence of nested simplicial complexes:

\begin{equation}
\emptyset = Rips_0(X)  \subseteq \dots \subseteq Rips_{\epsilon_n}(X) = Rips_{\infty}(X).
\end{equation}
where \( \epsilon_1 < \epsilon_2 < \dots < \epsilon_n \). The persistence barcode of this filtration  consists of intervals \( (\alpha_b, \alpha_d) \), representing the birth and death of a fundamental topological feature (such as a connected component or loop) as the scale parameter \( \epsilon \) increases.

The set of features obtained can also be represented by the persistence diagram $PD[K(X)]$ in which each interval \( (\alpha_b, \alpha_d) \) is represented as a point in the $\Delta:=\{r=(r_1, r_2), r_1\le r_2\le \infty\}$ extended half-plane.

\subsection{The  density of expected persistence diagrams}

Let \(M\) be a compact, smooth \(d\)-dimensional Riemannian manifold.
A point cloud is modeled as the random vector
\begin{equation}
X=(X_1,\dots,X_n)\in M^n,
\qquad
X_i\stackrel{\text{i.i.d.}}{\sim}\mu.
\end{equation}
where \(\mu\) is a Borel probability measure on \(M\).
Replacement of deterministic samples with the random variable \(X\) allows bringing stochastic and statistical tools to the study of persistent homology.

For each scale parameter \(\alpha\ge 0\) we construct the simplicial complex
\(K(X,\alpha)\) and obtain the filtration
\(K(X)\).
The resulting \(s\)-dimensional persistence diagram $PD_s[K(X)]$ is random and we consider it as the discrete measure on $\mathbb{R}^2$:

\begin{equation}
D_s[K(X)] = \sum_{r\in PD_s[K(X)]} \delta_r.
\end{equation}
where $\delta_r$ is the Dirac measure concentrated at a  point $r$ from the persistence diagram $PD_s[K(X)]$. 

For a Borel set \(B\subset\Delta\), define the expectation of \(D_s[K(X)]\) in the standard way:
\begin{equation}
\mathbb{E}[D_s[K(X)]](B)=\mathbb{E}\bigl[D_s[K(X)](B)\bigr].
\end{equation}
The resulting deterministic measure \(\mathbb{E}[D_s[K(X)]]\) is called the
\emph{Expected Persistence Diagram}; it summarizes the average location of topological features generated by the sampling distribution \(\mu\).
  
In \cite{Density-Chazal}, it is proven that \(\mathbb{E}[D_s[K(X)]]\) admits a density \(p\) with respect to two-dimensional Lebesgue measure, i.e.  

\begin{equation}
\mathbb{E}[D_s[K(X)]](B)=\!\!\iint_{B}\! p(b,d)\,db\,dd,
\qquad B\subset\Delta.
\end{equation}
See Appendix~\ref{app:theorem_proof_base} for more details.



\subsection{Cross-Barcode (P, Q)}
\label{sec:Cross-Barcode}

The Cross-Barcode extends the concept of persistence to compare the topological features of two different point clouds, \( P \) and \( Q \), by examining their combined topological structure. To construct the Cross-Barcode, we first create a filtered simplicial complex based on the union of \( P \) and \( Q \). This is done by considering the weighted graph  (\( \Gamma_{P \cup Q}\), \(m_{(P \cup Q)/Q} \)), where the vertices represent the points in \( P \cup Q \), and the edge weights are the pairwise distances between these points. Notably, the distances between points within \( Q \) are set to zero, which emphasizes the structure of \( P \) relative to \( Q \).

The resulting Vietoris-Rips complex \( R_{\alpha}(\Gamma_{P \cup Q}, m_{(P \cup Q)/Q}) \) is a sequence of nested simplicial complexes, and the persistence of topological features is tracked as \( \alpha \) increases. Specifically, the \( \text{Cross-Barcode}_i(P, Q) \) records the persistence intervals for each \( i \)-dimensional feature in this combined filtration. These intervals capture the birth and death times of topological features that arise due to the interaction between the point clouds \( P \) and \( Q \), allowing for a direct comparison of their topological structures.

\begin{equation}
\text{Cross-Barcode}_i(P, Q) = \{ (\alpha_b, \alpha_d) \mid \alpha_b < \alpha_d \}.
\end{equation}

Topological features that persist across both point clouds are considered significant, and the Cross-Barcode helps quantify the extent to which the features of \( P \) and \( Q \) align in their topology. This concept is particularly useful in applications where comparing the topological similarity between two manifolds is essential. An example of a cross-persistence diagram is depicted in Fig.~\ref{fig:pipeline Rips}

\begin{figure}[ht]
\vspace{-.15in}
\centering
{\includegraphics[width=1\linewidth]{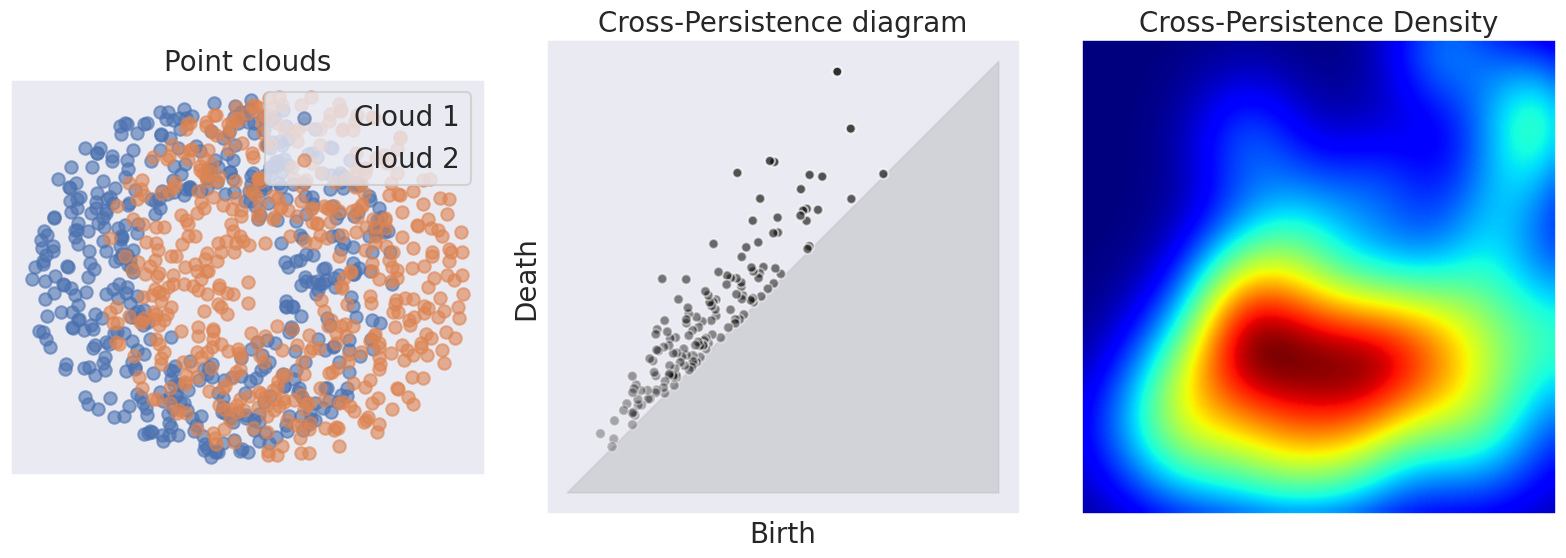}}
\caption{Standard pipeline for Cross-barcode vectorization}
\label{fig:pipeline Rips}
\end{figure}

\section{The density of the cross-persistence diagrams}
\label{sec:Density}
In this work, we extend the result of \cite{Density-Chazal} from persistence diagrams to cross-persistence diagrams.

 \begin{theorem}[]
 Let $n, k \geq 1$. Assume that $M$ and $N$ are real analytic compact $d$-dimensional connected submanifolds, possibly with boundaries, and that $\mathbf{X} \in M^n$, $\mathbf{Y} \in N^k$ are random variables with densities with respect to the Hausdorff measures $\mathcal{H}_{dn}$ and $\mathcal{H}_{dk}$, respectively. Then, for the $\mathcal{K}$ - Vietoris-Rips cross-persistence filtration, $\mathbf{Z} = (\mathbf{X}, \mathbf{Y})$ and $s \geq 1$, the expected measure $\mathbb{E}[D_s[\mathcal{K}(\mathbf{Z})]]$ admits a density with respect to the Lebesgue measure on $\Delta$. Furthermore, $\mathbb{E}[D_0[\mathcal{K}(\mathbf{Z})]]$ admits a density with respect to the Lebesgue measure on the vertical line $\{0\} \times [0, \infty)$.
 \end{theorem}
 \begin{proof}
See Appendix~\ref{app:theorem_proof_cross}.
\end{proof}

The space of persistence diagrams is inherently complex, motivating the use of vector space representations. A widely adopted class of mappings is linear representations, defined as  
\begin{equation}
\Psi(D_s) = D_s(f) := \sum_{r \in D_s} f(r).
\end{equation}
where $f$ is a function on $\Delta$. This framework encompasses several well-known representations, such as the persistence image~\cite{adams2016persistenceimages}, the Manifold Topology Divergence~\cite{barannikov2021manifold}, \rev{and} the persistence silhouette~\cite{chazal2013Silhouettes}, among others. These representations not only provide a practical embedding of diagrams into vector spaces, but also preserve essential topological information in a form suitable \rev{for} statistical and machine learning methods.

A key property that makes these representations particularly attractive is that they often admit well-defined probability densities, which are crucial for developing rigorous statistical frameworks. The existence of such densities for a broad class of linear representations is a well-established result.

\begin{proposition}
\label{th:linear_repr}
Under the same assumptions, the density of $\mathbb{E}[D_s[\mathcal{K}(\mathbf{X})](f)]$ exists for all commonly used linear representations $f$ on $\Delta$, which have particularly good theoretical properties.  
\end{proposition}

\begin{proof}
Once the cross-persistence diagram admits a Lebesgue density on $\Delta$, applying any linear functional $f$ places us in the standard persistence-diagram setting, so the existence result transfers directly.
\end{proof}


This theoretical guarantee allows us to utilize linear representations of cross-persistence diagrams within statistical pipelines, enabling, for instance, density estimation, hypothesis testing, and the application of classical statistical learning tools to topological summaries of data.

\section{Distinguishing Point Clouds via the Density of Cross-Persistence Diagrams (MTD)}
\label{sec:experiments}
MTD (MTop-Divergence) is a linear representation derived from Cross-persistence diagrams, introduced as a tool for evaluating deep generative models. This method quantifies the discrepancy between two point clouds using information encapsulated in the Cross-Barcode. It provides a practical approach for distinguishing point clouds based on their topological features.

Given two point clouds \( P \) and \( Q \), the MTD in homological dimension \( i \) is defined as:
\begin{equation}
\mathrm{MTD}_i(P, Q) := \sum_{(\alpha_b, \alpha_d) \in \mathrm{CrossBarcode}_i(P, Q)} (\alpha_d - \alpha_b).
\end{equation}

In this work, we address the problem of determining whether a subsample from an unknown point cloud originates from a given core point cloud or another in a collection. Given a set of point clouds \( Q_1, \dots, Q_n \subset \mathbb{R}^n \), each of varying size, \( Q_1 \) is designated as the core cloud. Our goal is to classify a subsample \( \hat{Q_s} \) as originating from \( Q_1 \) or from another cloud in the set.

To solve this, we introduce a statistical method based on MTD. We first estimate the density of MTD values (the existence of this density is established in Proposition~\ref{th:linear_repr}) for subsamples from \( Q_1 \), capturing the distribution of MTD values when comparing a point cloud to itself. Next, we compute the MTD value between \( Q_1 \) and the unknown subsample \( \hat{Q_s} \). By evaluating this value against the density of \( MTD(Q_1, Q_1) \), we estimate the probability that \( \hat{Q_s} \) belongs to \( Q_1 \). This probabilistic framework allows us to decide the origin of \( \hat{Q_s} \) with a high degree of confidence.


To evaluate the performance of our method, we estimate the density of \( \mathrm{MTD}(Q_1, Q_s) \) for each point cloud in the dataset. This is done by computing MTD scores between multiple random subsamples drawn from the core cloud \( Q_1 \) and each candidate cloud \( Q_s \). The resulting empirical density reflects the distribution of MTD values when comparing \( Q_1 \) to other clouds. By analyzing the degree of overlap between \rev{the} density \( \mathrm{MTD}(Q_1, Q_s) \) and \rev{the} reference density \( \mathrm{MTD}(Q_1, Q_1) \), we estimate the probability that the given subsample \( \hat{Q_s} \) originates from the same distribution as \( Q_1 \). A low degree of overlap indicates dissimilar topology. Notably, the order of arguments in \( \mathrm{MTD}(\cdot, \cdot) \) is asymmetric, and reversing the order may lead to different results. This asymmetry is an intrinsic property of the method and must be considered when interpreting the results.

This approach demonstrates the utility of MTD as a robust tool for differentiating point clouds, paving the way for further applications in evaluating generative models and analyzing topological discrepancies.

\rev{
A theoretical justification of the robustness of the density-based comparison procedure, including stability of the MTD density and of the density-overlap functional with respect to estimation error, is provided in Appendix~\ref{app:mtd_stability}.
}

\subsection{Experiments}

For the experiments, we used four real-world datasets containing several point clouds, that is, groups of objects that can be classified into different categories. First, we illustrated the ability of our proposed method to distinguish between point clouds with different intrinsic geometries using a simple dataset (MNIST, Fig.~\ref{fig:mnist}). Next, we demonstrated that our algorithm also performs well on more complex datasets, although it may struggle in certain cases (CIFAR-10 and COIL-20, Figs.~\ref{fig:cifar10} ,~\ref{fig:coil20} and Fig.~\ref{fig:MTD_no_noise_exp_failed}). To overcome these challenges, we proposed injecting noise into the image datasets to improve the quality of the results (MNIST, CIFAR-10, and COIL-20, Fig.~\ref{fig:MTD_noise_exp}).\rev{Finally, we evaluated the method in a high-complexity regime with a larger number of categories (CIFAR-100), where direct visual inspection of density overlap becomes less informative, yet substantial geometric differences between the distributions persist (Fig.~\ref{fig:CIFAR100_mtd}).}


\subsubsection{Distinction with simple test}
As discussed earlier, for each point cloud $Q_k$ in the \rev{considered} dataset, we estimated the density of $\mathrm{MTD}(Q_k, Q_k)$ and $\mathrm{MTD}(Q_k, Q_i)$ for all other point clouds. We then placed them on one graph to be able to visually distinguish them, see Fig.~\ref{fig:MTD_no_noise_exp}. As can be seen from this figure, our proposed method is able to capture the difference between samples from a core cloud and samples from other clouds, as these densities barely intersect.

\begin{figure}[t]
\vspace{-0.15in}
\centering
\subfloat[MNIST]{%
    \includegraphics[width=0.32\linewidth]{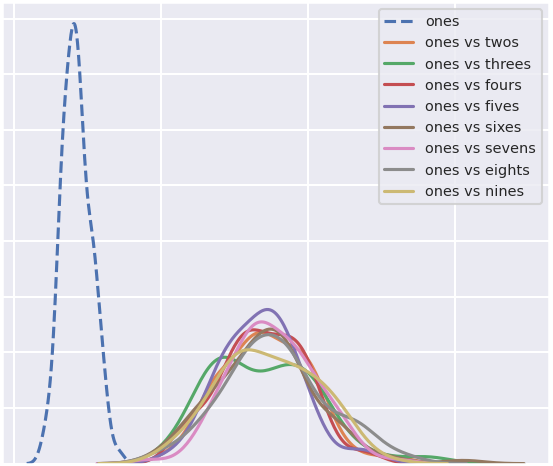}%
    \label{fig:mnist}}
\hfil
\subfloat[CIFAR10]{%
    \includegraphics[width=0.32\linewidth]{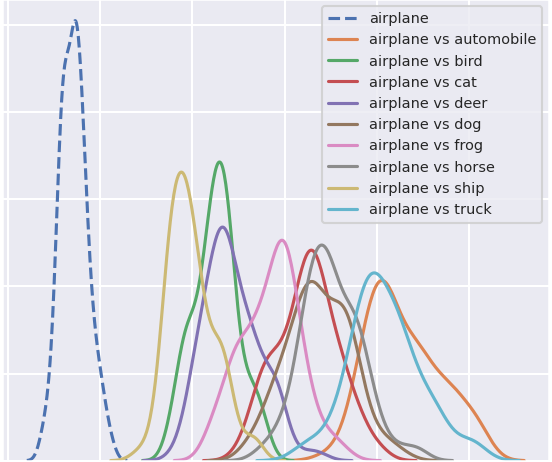}%
    \label{fig:cifar10}}
\hfil
\subfloat[COIL20]{%
    \includegraphics[width=0.32\linewidth]{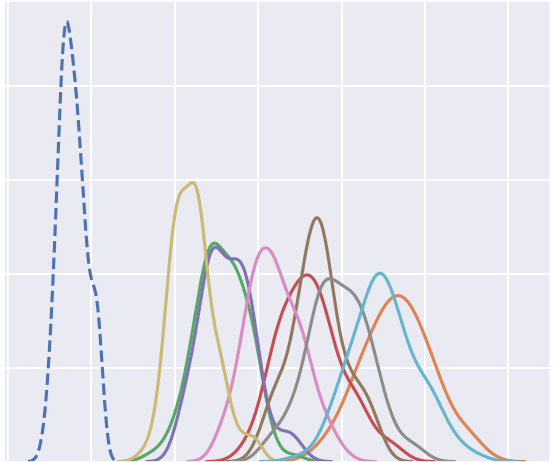}%
    \label{fig:coil20}}
\caption{In these pictures, the density of $MTD(Q_1, Q_1)$ is represented by dashed lines, and all other densities are represented by continuous lines. For each dataset, there is only one picture presented with one core cloud.}
\label{fig:MTD_no_noise_exp}
\end{figure}
Despite very good results on the MNIST dataset, \rev{the overlap-based comparison becomes less visually separable for a subset of pairs in the CIFAR10 and COIL20 datasets (approximately 40\%), see Fig.~\ref{fig:MTD_no_noise_exp_failed}. In these cases, the density supports exhibit non-negligible intersection, indicating a more challenging discrimination regime.}

\begin{figure}[t]
\vspace{-0.15in}
\centering
\subfloat[CIFAR10]{%
    \includegraphics[width=0.49\linewidth]{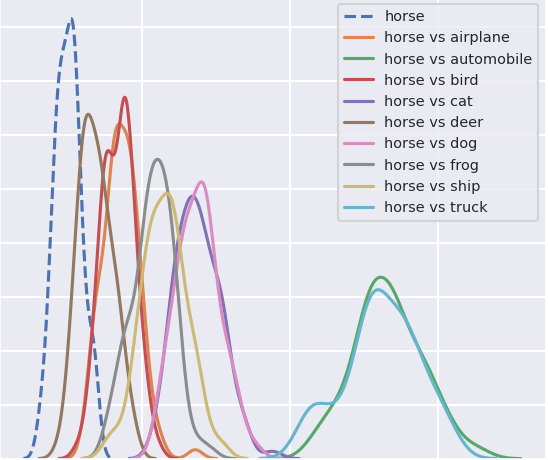}%
    \label{fig:cifar10_failed}}
\hfil
\subfloat[COIL20]{%
    \includegraphics[width=0.49\linewidth]{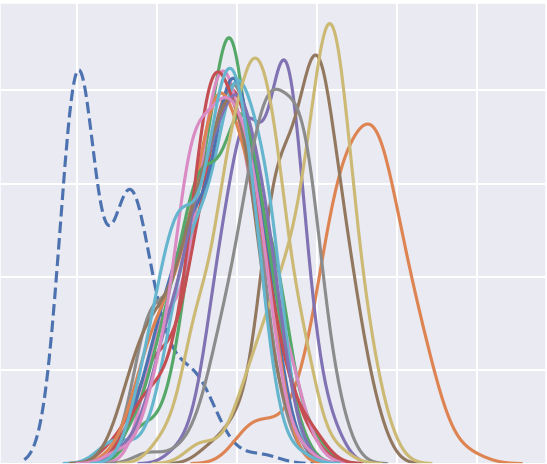}%
    \label{fig:coil20_failed}}
\caption{Densities of $MTD(Q_1, Q_1)$ (dashed) and $MTD(Q_1, Q_s)$ (solid). As can be seen from these images, the probability that our approach suggests that the value of $MTD(Q_1, Q_k)$ was sampled from the same distribution  is greater than the classical significance threshold (0.05).}
\label{fig:MTD_no_noise_exp_failed}
\end{figure}

\subsubsection{Applying noise for amplification of accuracy}
Relative articles from the computer vision domain often utilize noise to improve the performance of their algorithms~\cite{LI2024103855, Parametric_noise}. Such approaches have been shown to enhance robustness, promote better generalization, and mitigate overfitting. Moreover, noise injection helps models to focus on informative structures in the data, leading to the discovery of more discriminative and semantically meaningful features.

In order to overcome the difficulties that our approach encountered in the CIFAR-10 and COIL-20 datasets, we decided to add noise to them in varying degrees, see Fig.~\ref{fig:MTD_noise_exp} and~\ref{fig:MTD_noise_process}. In all these experiments only right point cloud in $MTD(Q_1, Q_i)$ was noised for all $i$.

\begin{figure*}[!t]
\vspace{-0.15in}
\centering
\subfloat[CIFAR10]{%
    \includegraphics[width=1\linewidth]{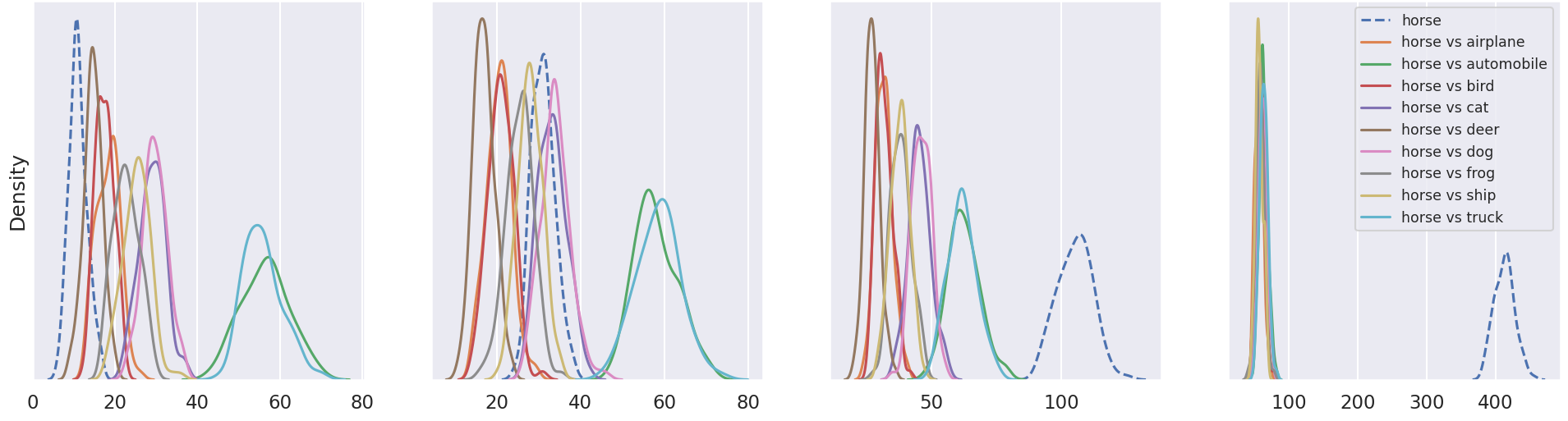}%
    \label{fig:cifar10_noise}}
\vfill
\subfloat[COIL20]{%
    \includegraphics[width=1\linewidth]{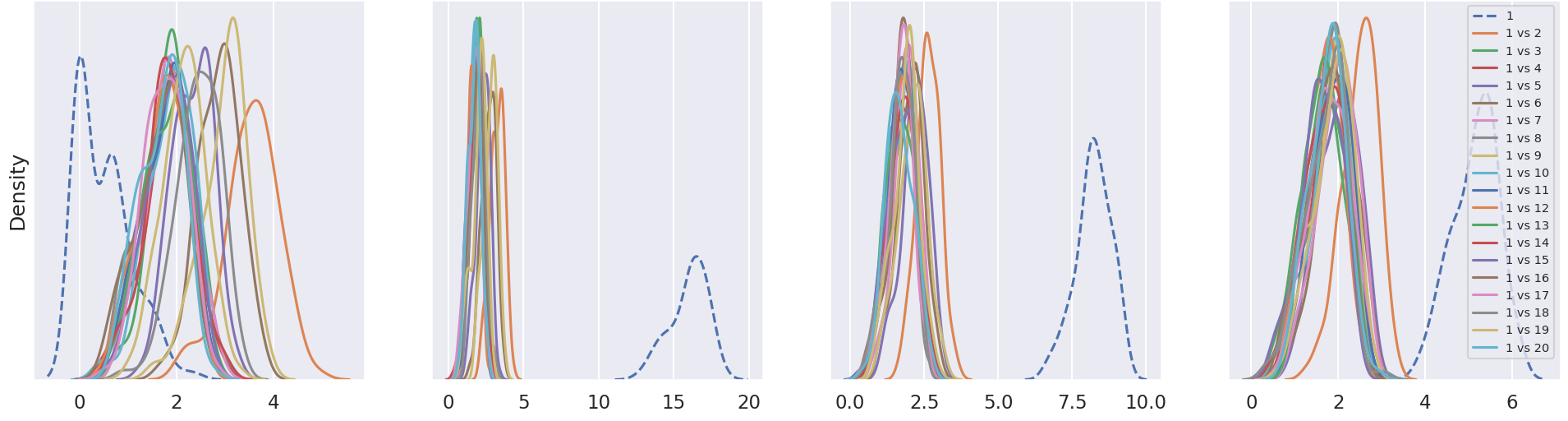}%
    \label{fig:coil20_noise}}
\caption{Densities of $MTD(Q_1, Q_1)$ (dashed) and $MTD(Q_1, Q_s)$ (solid), where Gaussian noise is applied only to the right argument (i.e., $MTD(\text{Pure}_i, \text{Noised}_i)$). From left to right, the relative noise norms $||\xi|| / ||x||$ are $[0\%, 25\%, 50\%, 75\%]$. As the noise intensity increases, the density of $MTD(Q_1, Q_1)$ gradually shifts to the right and eventually merges with the others once the objects become indistinguishable.}
\label{fig:MTD_noise_exp}
\end{figure*}
\begin{figure}[ht]
\vspace{-.15in}
\centering
{\includegraphics[width=1\linewidth]{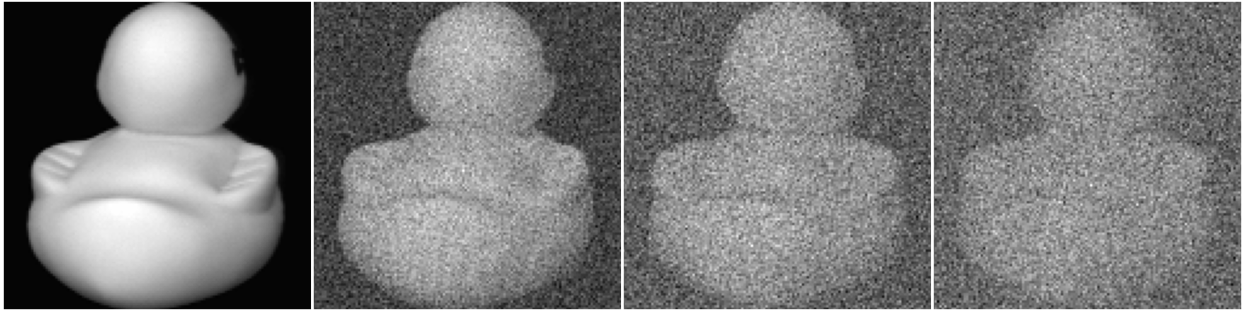}}
\caption{The process of adding Gaussian noise to COIL20 images. From left to right, the relative noise levels $||\xi|| / ||x||$ are $[0\%, 25\%, 50\%, 75\%]$}
\label{fig:MTD_noise_process}
\end{figure}

With the intention of understanding how the noise application process affects our algorithm, we conducted experiments using different positions for the application of noise: for all experiments, we added noise only to the right cloud and to both left and right. See Fig.~\ref{fig:MTD_noise_diff_places} for the results. From these experiments, we can conclude that MTD density is very robust and is not significantly affected when we apply noise to both sides of $MTD(*, *)$ simultaneously.

\rev{
Importantly, the application of noise does not introduce a new comparison criterion but rather amplifies geometric discrepancies already present in the cross-persistence structure. The density-overlap functional remains unchanged; noise shifts the relative mass of the distributions in a way that makes the existing discrepancy more pronounced and therefore easier to detect.}

\rev{
A quantitative sensitivity analysis of this effect is provided in Appendix~\ref{app:noise_sensitivity}, where we show that for datasets such as COIL20 there exists a broad range of noise levels for which the average overlap between $MTD(Q_1,Q_1)$ and $MTD(Q_1,Q_k)$ (for $k\neq 1$) approaches zero.
}

\begin{figure}[ht]
\vspace{-0.15in}
\centering
\subfloat[MNIST]{%
    \includegraphics[width=1\linewidth]{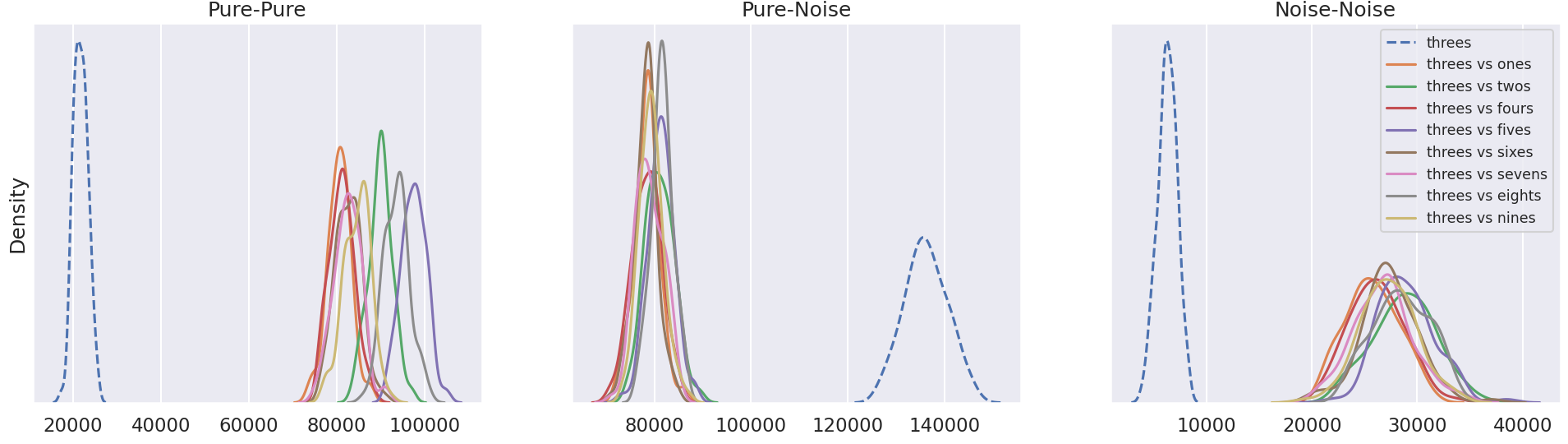}%
    \label{fig:mnist_noise_diff}}
\vfill
\subfloat[CIFAR10]{%
    \includegraphics[width=1\linewidth]{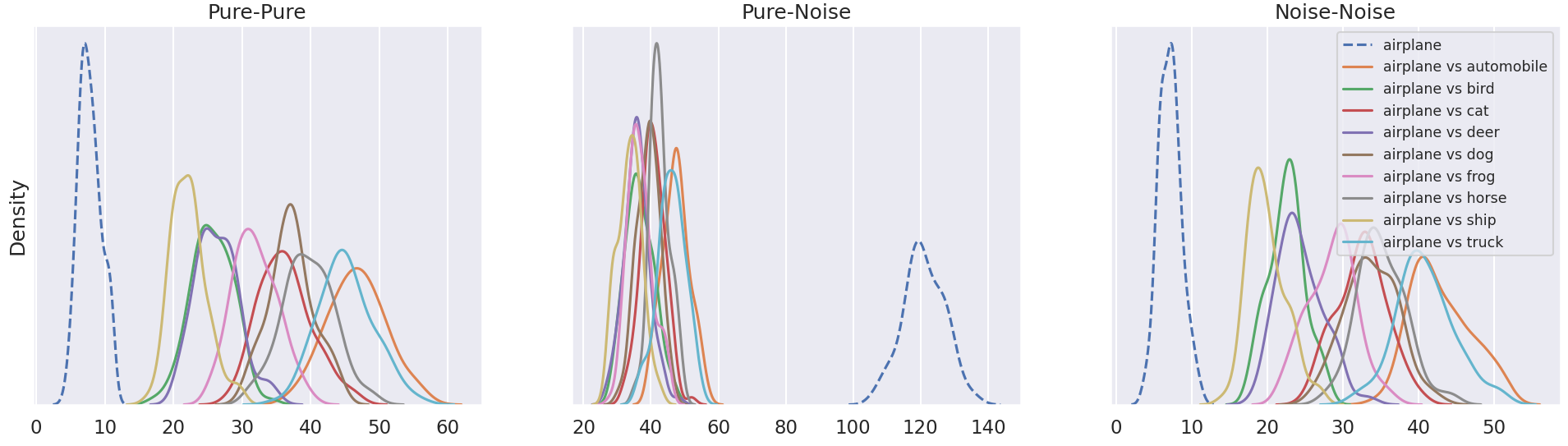}%
    \label{fig:cifar10_noise_diff}}
\vfill
\subfloat[COIL20]{%
    \includegraphics[width=1\linewidth]{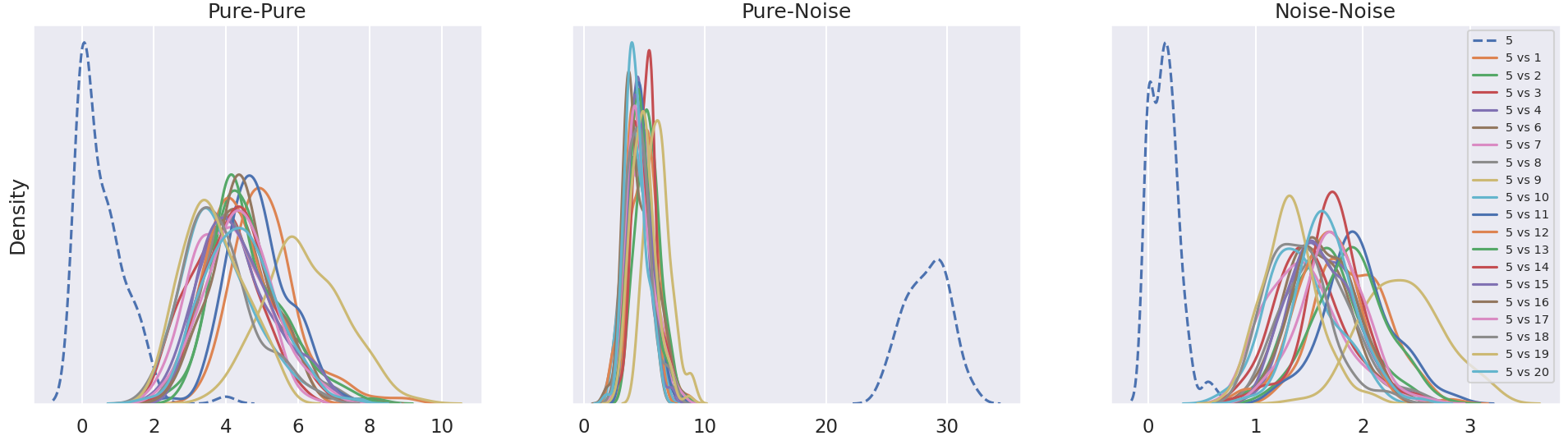}%
    \label{fig:coil20_noise_diff}}
\caption{Densities of $MTD(Q_1, Q_1)$ (dashed) and $MTD(Q_1, Q_s)$ (solid). The relative noise norm is $||\xi|| / ||x|| = 50\%$. These plots illustrate three regimes of point cloud comparison: without noise, with noise applied only to the right point cloud, and with noise applied to both.}
\label{fig:MTD_noise_diff_places}
\end{figure}




\subsubsection{High-complexity regime (CIFAR100)}
\rev{
In the case of the CIFAR100 dataset, the overlap-based comparison becomes significantly more challenging, even when noise is applied, see Fig.~\ref{fig:CIFAR100_mtd}. Nevertheless, the overlap values remain far from saturation, indicating substantial discrepancy between $MTD(Q_1,Q_1)$ and $MTD(Q_1,Q_s)$ distributions.}

\rev{In this regime, direct visual inspection of density overlap becomes less informative due to the high complexity and heterogeneity of the dataset. However, we observe that the support of $MTD(Q_1,Q_1)$ is consistently more concentrated than that of $MTD(Q_1,Q_s)$, reflecting reduced geometric variability when comparing a point cloud with itself.}

\rev{To better characterize this effect, we analyze the relative dispersion of the MTD distributions as a descriptive second-order property that is visually apparent in the estimated densities. 
Importantly, dispersion statistics are not used as a decision criterion in our method; the overlap functional remains the sole comparison metric. Dispersion is reported only as a descriptive second-order property that helps interpret the geometric structure of the distributions in high-complexity regimes.}


\rev{
This situation corresponds precisely to the scenario illustrated in Appendix~\ref{app:overlap_metric}, Fig.~\ref{fig:overlap_four_cases} (Case~3), where two densities exhibit substantial overlap while remaining far from identical and differing markedly in their dispersion. In this regime, the non-saturated overlap already indicates strong discrepancy, while dispersion statistics serve to contextualize the geometric structure of the overlap region.
}

\begin{figure}[ht]
\vspace{-0.1in}
\centering
\subfloat[CIFAR100]{%
    \includegraphics[width=\linewidth]{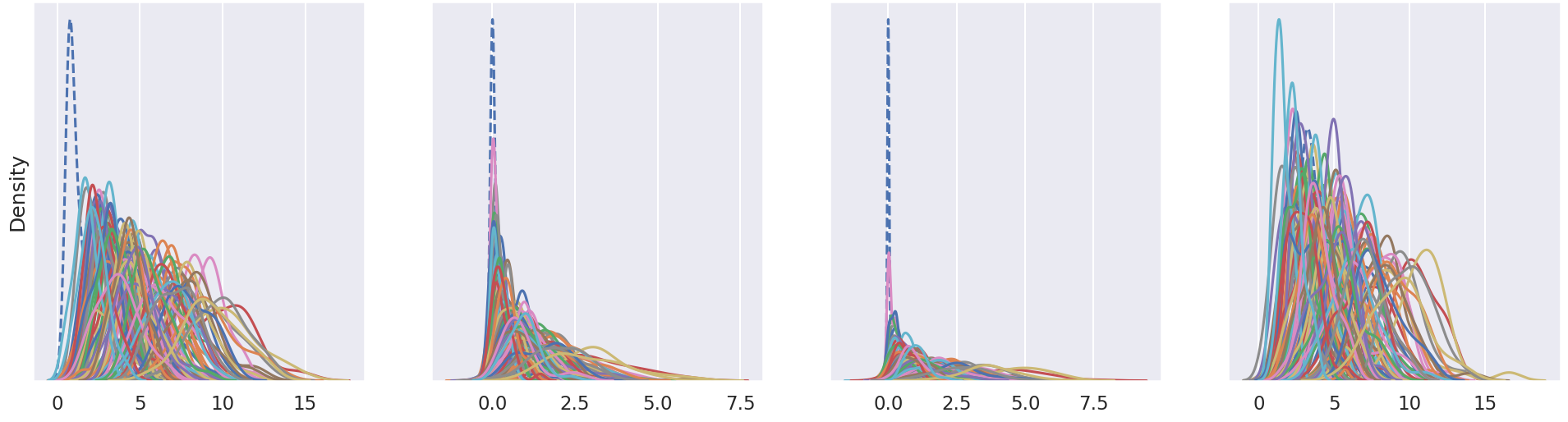}%
    \label{fig:cifar100_mtd}}
\caption{Densities of $MTD(Q_1, Q_1)$ (dashed) and $MTD(Q_1, Q_s)$ (solid), where Gaussian noise is applied only to the right argument (i.e., $MTD(\text{Pure}_i, \text{Noised}_i)$). From left to right, the relative noise norms $||\xi|| / ||x||$ are $[0\%, 25\%, 50\%, 75\%]$. In this high-complexity regime, direct overlap comparison becomes less visually informative due to substantial distributional heterogeneity. However, the overlap values remain far from saturation, indicating persistent geometric discrepancy. The relative  concentration (dispersion) of $MTD(Q_1,Q_1)$ compared to $MTD(Q_1,Q_s)$ 
provides additional descriptive insight into geometric variability. 
See Appendix~\ref{app:overlap_metric} for discussion of this regime.}
\label{fig:CIFAR100_mtd}
\end{figure}

\section{Cross-RipsNet}
\label{sec:Cross-RipsNet}

In the previous sections, we explored the use of cross-persistence in statistical inference. Specifically, we focused on the density of a linear representation derived from these diagrams, namely the Manifold Topology Divergence (MTD), which summarizes diagram information into a single interpretable value. While this approach provides an effective tool for comparing point clouds, estimating such densities requires repeated computation of cross-persistence diagrams-a computationally intensive process.

The RipsNet model, introduced in~\cite{desurrel2022ripsnetgeneralarchitecturefast}, is a point-order invariant neural network designed to transform point clouds into persistence images without directly computing topological features. Although effective in some contexts, it performs poorly in our setting, see Table~\ref{tab:dim_red_comparing}, as predicting cross-persistence densities requires processing two distinct point clouds and capturing their joint topological structure. To adapt RipsNet for this task, we replaced the original loss function with the Kullback--Leibler divergence, which better reflects the distributional nature of our output. We also introduced a normalization layer to improve stability, since the original model failed to converge without these modifications.

To address the computational bottleneck associated with repeated diagram estimation, we introduce \textit{Cross-RipsNet}-a neural architecture designed to predict either the density of cross-persistence diagrams or the density of linear representation such as MTD directly from raw point cloud data. In Section~\ref{sec:Cross_RipsNet_experiments}, we apply Cross-RipsNet to approximate the density of cross-persistence diagrams, while in Appendix~\ref{app:MTD_density_prediction}, we demonstrate how it can be adapted to estimate the density of MTD values.

\rev{In relation to prior persistence-learning architectures, \textit{Cross-RipsNet} differs in both input design and objective. Unlike approaches that operate directly on persistence diagrams with permutation-invariant layers, \textit{Cross-RipsNet} predicts the density of cross-persistence diagrams from raw point clouds (and, optionally, distance-derived features), avoiding explicit diagram computation at inference. Relative to RipsNet, which encodes a single cloud into persistence images, \textit{Cross-RipsNet} explicitly models paired clouds via separate encoders and a shared head, and can incorporate the asymmetric distance matrix \(m_{(P \cup Q)/Q}\) to capture cross-structure that cannot be recovered from either cloud alone. These departures are motivated by the asymmetry and pairwise nature of cross-persistence and by the need to predict a distributional object rather than a single diagram.}

This approach preserves the topological expressiveness of persistence-based descriptors while significantly reducing the computational cost of applying them at scale.

\subsection{Two types of Cross-RipsNet}

The RipsNet model was designed to use coordinate matrices \( X = \{x_1, \dots, x_n\} \subset \mathbb{R}^d \) for predicting persistent images. To achieve this, the authors used a point-order-invariant layer based on the DeepSets architecture, which transforms a given matrix into a feature vector.  
\begin{equation}
\mathrm{RipsNet}: X \rightarrow \phi_2 \left( \sum_{x \in X} \phi_1(x) \right).
\end{equation}
where \( \phi_1: \mathbb{R}^d \rightarrow \mathbb{R}^{d'} \) and \( \phi_2: \mathbb{R}^{d'} \rightarrow \mathbb{R}^{d''} \). See \cite{zaheer2018deepsets} for more information.

To extend this approach for predicting density of cross-persistent diagrams, that can be considered as averaged cross-persistence images, we proposed three architectural variants of Cross-RipsNet, described below and illustrated in Fig.~\ref{fig:Pic_of_archs}.

\textbf{(a) Modified RipsNet.} The first approach treats two input point clouds as a single, merged cloud by concatenating their coordinate matrices. The resulting set is processed jointly through the RipsNet pipeline, as illustrated in Fig.~\ref{fig:Pic_of_archs}a.

\textbf{(b) Cross-RipsNet.} In the second approach, additionally each point cloud is processed independently. The resulting feature representations are then combined and passed through a shared network head. This design, shown in Fig.~\ref{fig:Pic_of_archs}b, enables the model to learn interactions between clouds while maintaining separate encodings.

\textbf{(c) Cross-RipsNet with Distance Matrix.} The third approach builds upon the second, but additionally incorporates the asymmetric distance matrix \( m_{(P \cup Q)/Q} \) between the two clouds \( P \) and \( Q \) (see Section~\ref{sec:Cross-Barcode}). After applying a dimensionality reduction technique to this matrix, it is processed via a DeepSets-based module alongside the outputs of the individual point clouds, see Fig.~\ref{fig:Pic_of_archs}c. This approach inherits the asymmetry from the distance matrix \(m_{(P \cup Q)/Q} \), which enables the model to capture more effectively the underlying structural dependencies in the data.

Empirically, the third approach, Cross-RipsNet with Distance Matrix, demonstrated significantly better performance across our experiments, see Table~\ref{tab:dim_red_comparing}.

\begin{figure}[ht]
\vspace{-.15in}
\centering
{\includegraphics[width=1\linewidth]{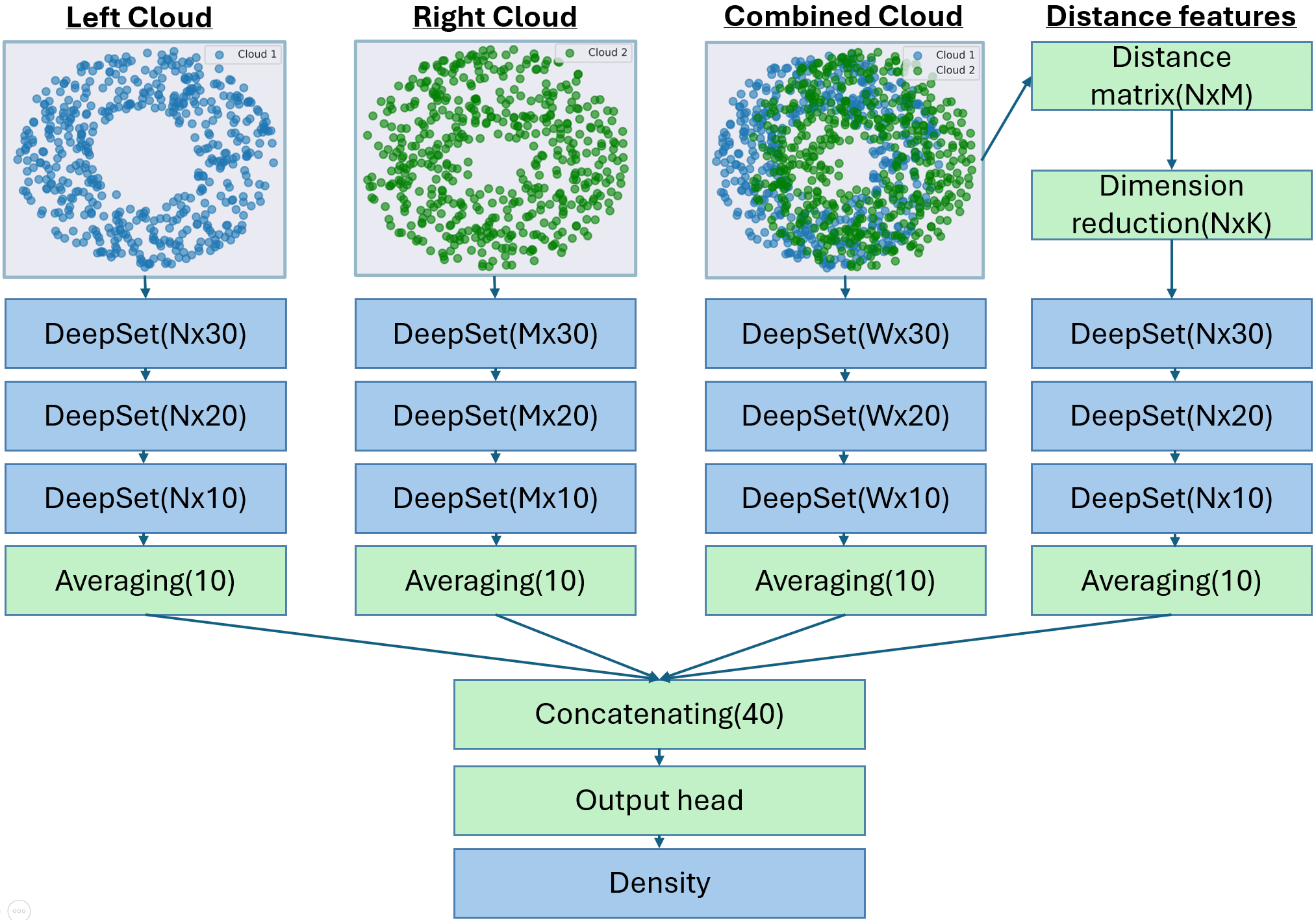}}

\caption{Illustration of the proposed architectures. Let \( N \) and \( M \) denote the sizes of the input point clouds, and let \( W = N + M \). The parameter \( K \) represents the output dimensionality of the dimension reduction method applied to the distance matrix, which contains zeros for all pairs of points within the right cloud. (a) \textbf{RipsNet} includes only the combined cloud block. (b) \textbf{Cross-RipsNet} includes all blocks except the distance features. (c) \textbf{Cross-RipsNet with distance matrix} incorporates all input blocks, including both point coordinates and distance-based features. See Appendix \ref{app:Ablation_CrossRpsNet} for the ablation study of the inputs.}

\label{fig:Pic_of_archs}
\end{figure}

\subsection{How to extract information from distance matrix}
In classical straightforward algorithms, one must construct all possible graphs with rising thresholds under point clouds to calculate persistence diagrams. For this purpose, it is necessary to compute the distance matrix between all points. This results in a computationally expensive process, as the size of the distance matrix grows quadratically with the number of points. As a result, the straightforward method of transforming point clouds into their persistence diagrams is based solely on the distance matrix; however, the RipsNet paper does not use this approach, opting for more efficient methods.

In Cross-RipsNet, we introduce a specialized layer to transform incoming distance matrix \(m_{(P \cup Q)/Q} \), derived from point clouds $P$ and $Q$, into a feature vector. While distance matrices can be computationally inefficient due to their large memory and processing requirements, we propose three approaches to extract essential information while reducing their dimensionality.

\begin{itemize}
    \item \textbf{Principal Component Analysis (PCA).} The first approach uses PCA to reduce the dimensionality of the asymmetric distance matrix to a fixed size \( K \). This provides a compact, low-rank representation of the distance structure, making the downstream neural model more memory-efficient while retaining the global geometry relevant for cross-persistence diagram estimation.

    \item \textbf{Top-\( K \) maximal distances.} In the second approach, we retain, for each point, only the \( K \) largest distances to other points. This is based on the observation that in cross-persistence constructions, prominent topological features often emerge from interactions between distant point pairs. Focusing on these maximal distances preserves critical information about global connectivity while significantly reducing matrix size.

    \item \textbf{Quantile-based distance summarization.} The third and most geometrically informed method involves selecting \( K \) quantiles from each point’s distribution of distances to others. This captures the statistical spread and local-to-global structure of each point’s neighborhood, providing a balanced and informative summary of topological context. As our experiments show, this approach achieves the best trade-off between compression and predictive performance.
\end{itemize}

\begin{table}[!t]
  \caption{Comparison of methods for predicting the density of cross-persistence diagrams across three domains. The $RipsNet$ and $Cross-RipsNet$ models do not use the distance matrix as input, while the other methods reduce its dimensionality to $K=60$ using different approaches.}
  \label{tab:dim_red_comparing}
  \centering
  \setlength{\tabcolsep}{4pt} 
  \renewcommand{\arraystretch}{1.2} 
  \footnotesize
  \resizebox{\columnwidth}{!}{ 
  \begin{tabular}{lccc}
    \toprule
    \textbf{Model} & \textbf{Synthetic} & \textbf{3D shapes} & \textbf{Textual data} \\
    \midrule
    \multicolumn{4}{c}{\textit{Symmetrical KL distance averaged among test data}} \\
    \midrule
    RipsNet & $2.60\pm0.10$ & $0.69\pm0.01$ & $0.90\pm0.03$ \\
    Cross-RipsNet & $2.34\pm0.05$ & $0.52\pm0.01$ & $0.62\pm0.03$ \\
    Cross-RipsNet/PCA & $2.40\pm0.04$ & $0.57\pm0.01$ & $0.66\pm0.02$ \\
    Cross-RipsNet/MAX & $2.38\pm0.04$ & $0.54\pm0.02$ & $0.66\pm0.05$ \\
    Cross-RipsNet/Quantiles & \textbf{2.05$\pm$0.08} & \textbf{0.48$\pm$0.03} & \textbf{0.56$\pm$0.01} \\
    \midrule
    \multicolumn{4}{c}{\textit{Symmetrical KL distance averaged among train data}} \\
    \midrule
    RipsNet & $1.99\pm0.04$ & $0.59\pm0.02$ & $0.27\pm0.012$ \\
    Cross-RipsNet & $1.85\pm0.15$ & $0.38\pm0.02$ & $0.21\pm0.001$ \\
    Cross-RipsNet/PCA & $2.15\pm0.11$ & $0.43\pm0.05$ & $0.21\pm0.004$ \\
    Cross-RipsNet/MAX & $1.81\pm0.20$ & $0.40\pm0.01$ & $0.22\pm0.001$ \\
    Cross-RipsNet/Quantiles & \textbf{1.38$\pm$0.06} & \textbf{0.35$\pm$0.01} & \textbf{0.21$\pm$0.005} \\
    \bottomrule
  \end{tabular}
  }
\end{table}

These methods are designed to preserve the critical information of the original distance matrix while making the computational process more manageable. As we can see in Table~\ref{tab:dim_red_comparing}, \rev{quantile-based summarization} outperforms other methods as a feature extractor from the distance matrix, showing promising results for use in Cross-RipsNet.

\rev{Despite the proposed dimensionality reduction strategies, Cross-RipsNet still requires the explicit construction of an asymmetric distance matrix as an intermediate step. On standard hardware (NVIDIA RTX~3060 with 32\,GB RAM), this limits the practical applicability of the model to point clouds of size up to $N\approx 10^4$ points per cloud, depending on the data dimensionality and batching strategy. Beyond this regime, the quadratic memory footprint of the distance matrix becomes the dominant bottleneck. Importantly, this limitation is shared by all methods relying on Vietoris-Rips type constructions and is not specific to Cross-RipsNet.}

\rev{Computationally, let $N=|P|$ and $M=|Q|$. Forming the asymmetric distance matrix $m_{(P \cup Q)/Q}$ costs $O(N  M d)$ time and $O(NM)$ memory, where $d$ is the ambient dimension used to compute distances, while the DeepSets-style encoders scale linearly in $N$ and $M$. Hence the dominant cost is quadratic in point counts, and higher-dimensional data tightens the practical limit through distance computation.}

\subsection{Cross-RipsNet experiments}
\label{sec:Cross_RipsNet_experiments}

In our experiments, we used one synthetic and two real datasets, which, as in the previous part, consist of objects divided into groups (classes). The goal was to evaluate the ability of Cross-RipsNet to predict the density of cross-persistence diagrams from point clouds, and to compare the performance of the model on both simple and complex datasets.

Firstly, we demonstrated the ability of Cross-RipsNet to predict the density of the cross-persistence diagrams using a simple synthetic dataset. This dataset consisted of randomly selected unions of circles in 2D space. It allowed us to evaluate the model on well-structured, low-dimensional data and assess its effectiveness in predicting cross-persistence diagram densities under ideal conditions.

Secondly, we worked with much more complex real-world data, including 3D shapes and text data. For the 3D shapes \cite{3D_shapes}, we sampled point clouds from various 3D objects. For the text-based data, we tokenized GPT- and human-generated answers \cite{Wiki_gpt} to prompts and embedded them using a pretrained Roberta model. This enabled us to test the model's robustness on high-dimensional, diverse data and assess its generalization capability.

The results demonstrate that Cross-RipsNet is able to accurately predict cross-persistence diagram densities, even for complex and high-dimensional point clouds. Moreover, the computational process is significantly faster than direct calculation methods, showing the potential of Cross-RipsNet for practical applications.

\subsection{Cross-RipsNet experiments (synthetic)}

For the synthetic dataset, we calculated cross-persistence densities directly from point clouds and trained a model to predict these images. The dataset included point clouds consisting of one, two, or three circles, and we explored three different combinations of these circles to generate cross-persistence densities. The combinations were as follows, see Fig.~\ref{fig:PI_examples_circles}:

\begin{enumerate}
    \item One circle vs. two circles
    \item One circle vs. three circles
    \item Two circles vs. three circles
\end{enumerate}

For each combination, we split the data into training and test sets, using 80\% of the data for training and 20\% for testing.

During these experiments, the base version of RipsNet was unable to predict the cross-persistence densities effectively for this dataset. However, with the modifications made to create Cross-RipsNet, we achieved a significant improvement in prediction quality, demonstrating the effectiveness of our model.
\begin{figure}[!t]
\centering
\subfloat[First class]{%
    \includegraphics[width=\columnwidth]{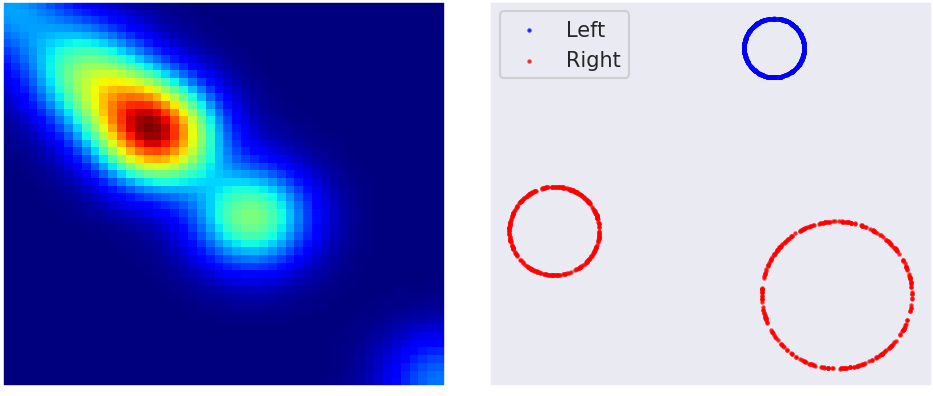}}\\
\subfloat[Second class]{%
    \includegraphics[width=\columnwidth]{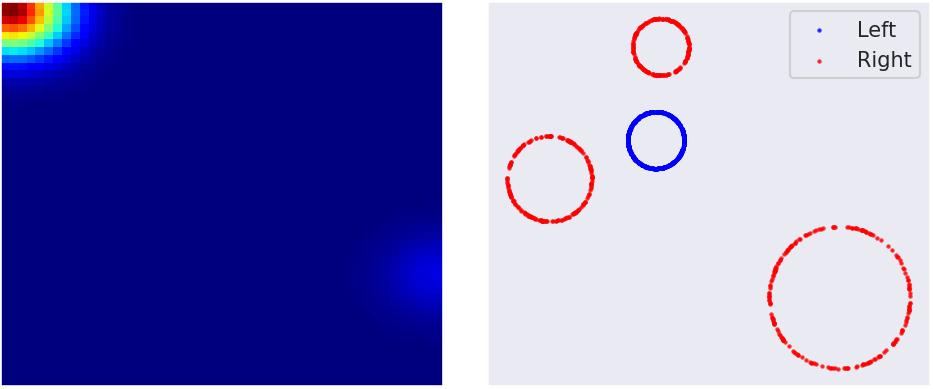}}\\
\subfloat[Third class]{%
    \includegraphics[width=\columnwidth]{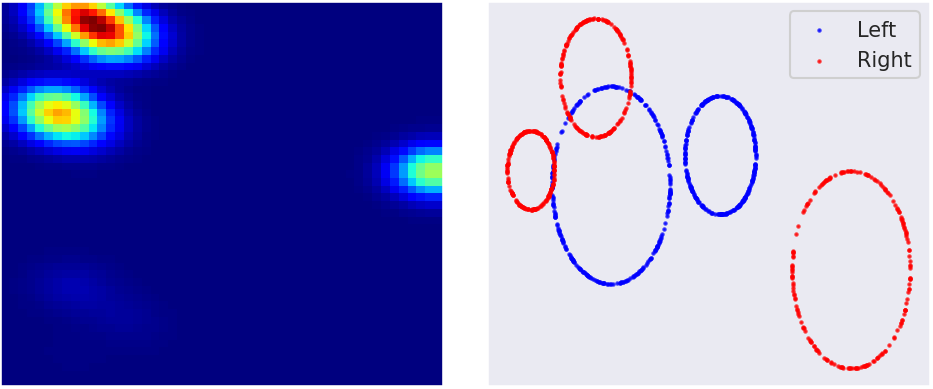}}
\caption{Examples of densities and corresponding pairs of point clouds for the cross-persistence diagrams in the synthetic dataset. All plots share the same spatial boundaries.}
\label{fig:PI_examples_circles}
\end{figure}

\subsection{Cross-RipsNet experiments (Real Data)}

For the real-world experiments, we used two datasets: 3D shapes (from the ModelNet10 dataset) and text data generated by GPT and humans.

\textbf{3D shapes (ModelNet10):} For each 3D object, we sampled 1024 points, which resulted in different point clouds corresponding to 10 distinct 3D shapes. We performed the same experiments as for the synthetic dataset, generating and predicting cross-persistence densities for these point clouds, see Fig.~\ref{fig:3d_shapes_results_test}. In the figure, the left side shows the real cross-persistence densities, while the right side shows the predicted cross-persistence densities. It can be observed that the Cross-RipsNet model effectively predicts the center of density, as well as the size and shape, of the cross-persistence densities.

\begin{figure}[t]
\centering
\subfloat[Train]{%
    \includegraphics[width=0.48\columnwidth]{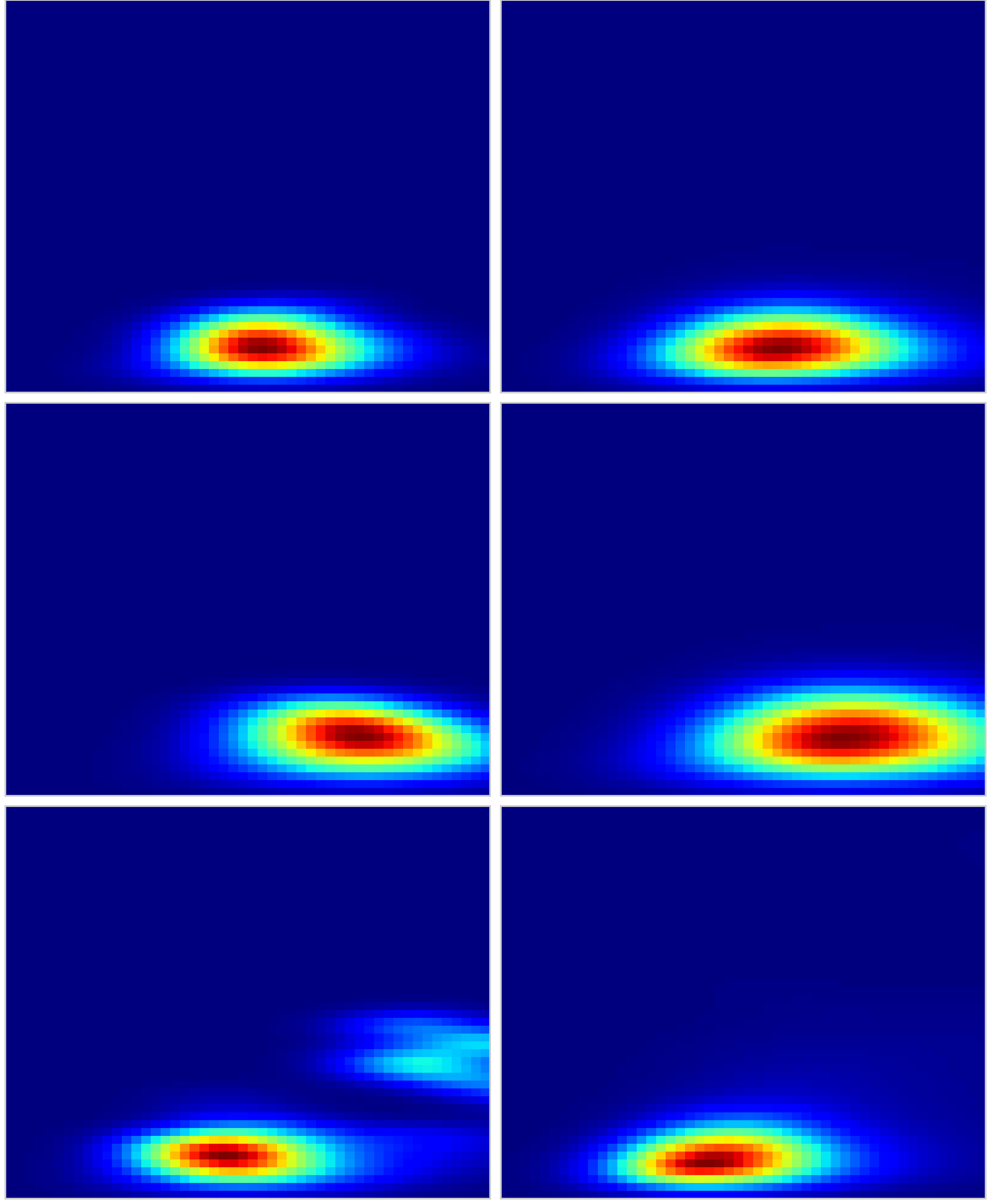}}
\hspace{0.02\columnwidth} 
\subfloat[Test]{%
    \includegraphics[width=0.48\columnwidth]{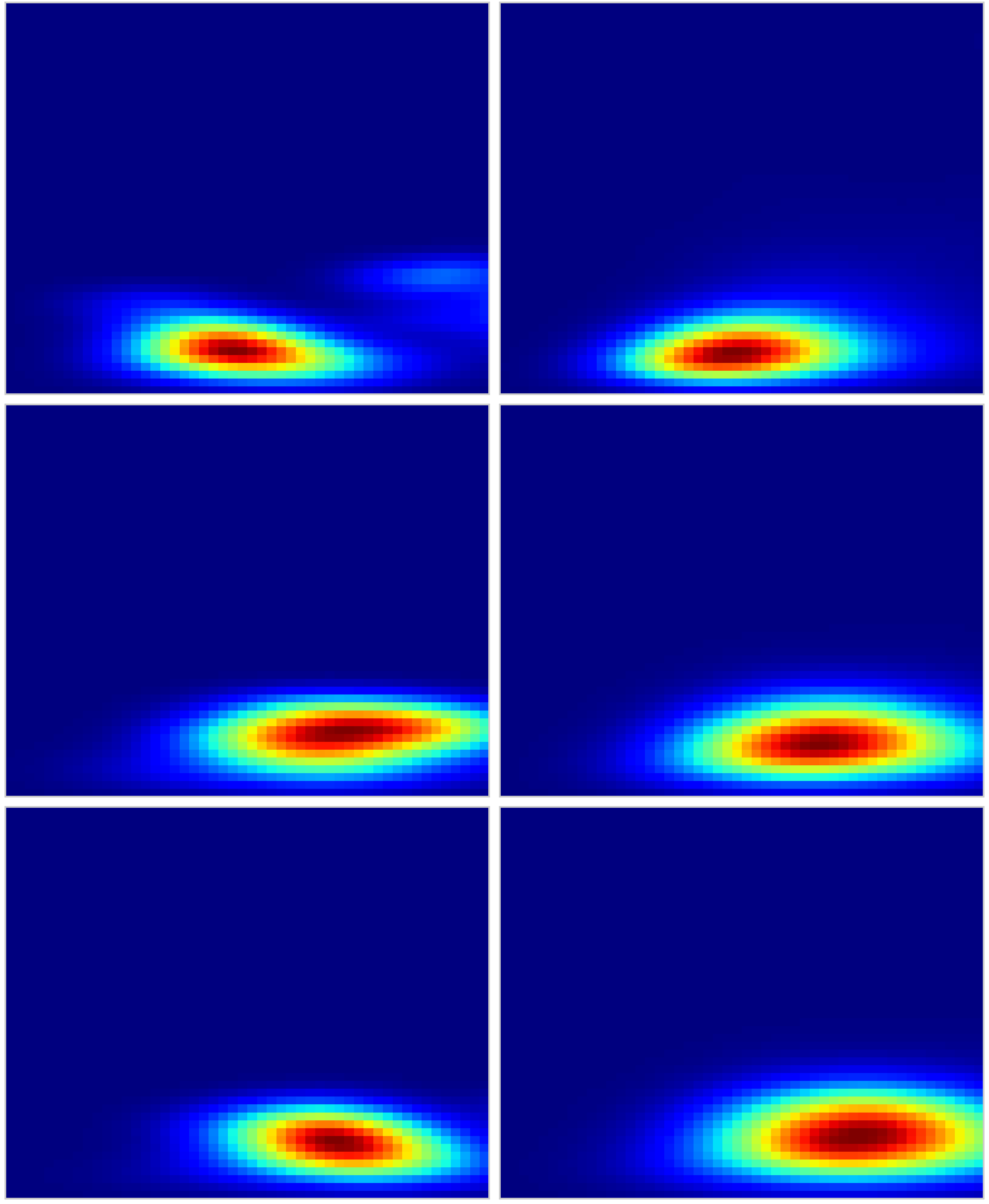}}
\caption{Results for the 3D shapes dataset. For each part of data, the left column shows the real persistence image, and the right column shows the predicted persistence image. Cross-RipsNet accurately predicts the center and shape of the density.}
\label{fig:3d_shapes_results_test}
\end{figure}

\textbf{GPT-Human generated answers on prompts:} For this experiment, we used data generated by both humans and GPT as their answers to prompts sourced from Wikipedia. We tokenized the text data and then used a pretrained Roberta model to transform these tokens into embedding vectors, which were subsequently used to generate point clouds. This resulted in 100 point clouds of different sizes, divided into two classes. Despite the varying point cloud sizes, Cross-RipsNet was able to handle the data effectively, as shown in Fig.~\ref{fig:Texts_results}. On the left side of the figure, we show the real cross-persistence densities, and on the right, the predicted cross-persistence densities. Cross-RipsNet demonstrated sensitivity to the size and shape of the density, as expected.

\begin{figure}[t]
\vspace{-.15in}
\centering
\subfloat[Train]
{\includegraphics[width=0.48\columnwidth]{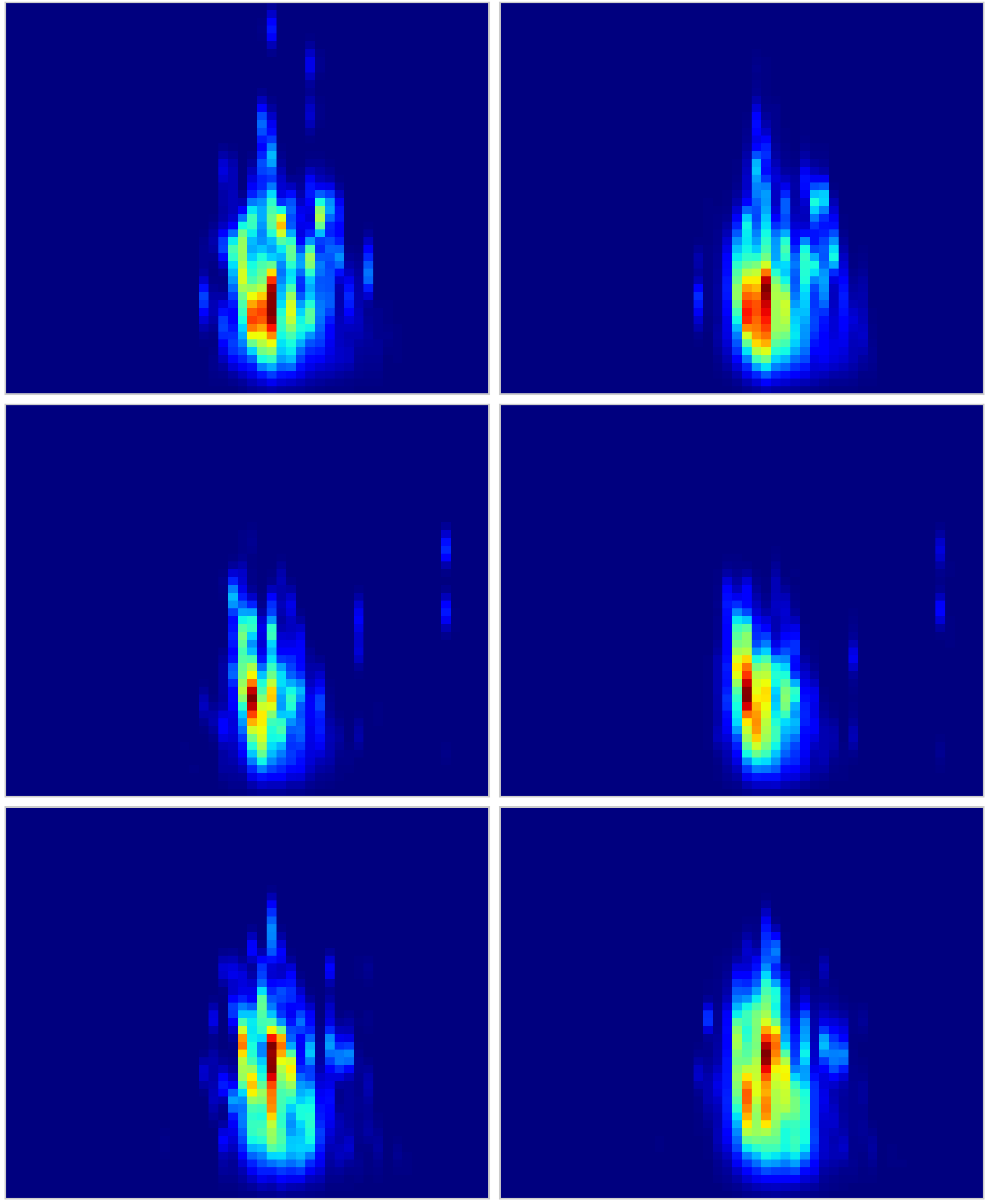}}
\hspace{0.02\columnwidth} 
\subfloat[Test]
{\includegraphics[width=0.48\columnwidth]{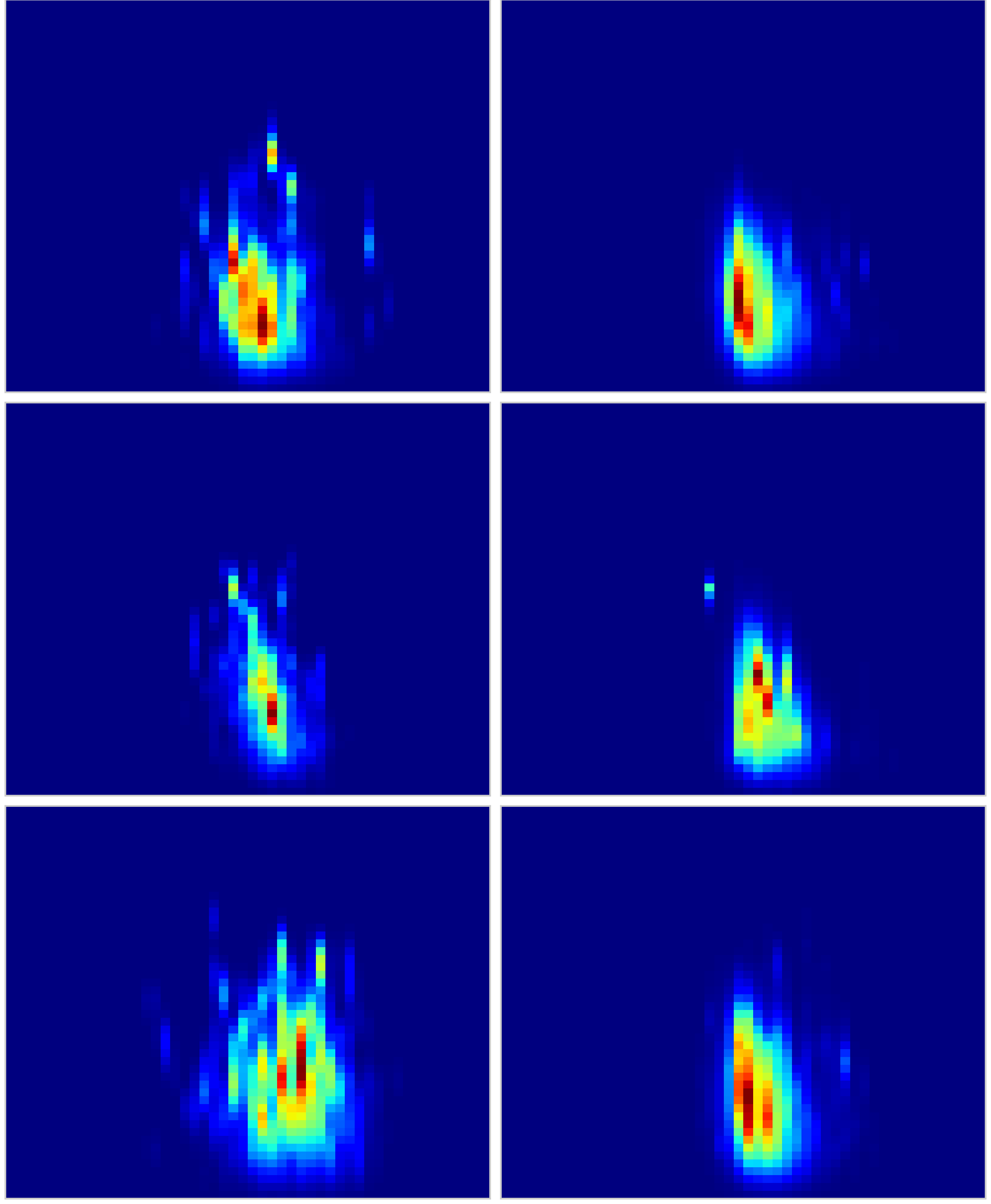}}
\caption{Results for text data. For each part of data, the left column shows the real persistence image, and the right column shows the predicted persistence image. The Cross-RipsNet model captures the size and shape of the density accurately.}
\label{fig:Texts_results}
\end{figure}

\textbf{Zero-dimensional features for text data:} In previous experiments, Cross-RipsNet worked with one-dimensional topological features. However, recent studies on the intrinsic dimension of generated data \cite{tulchinskii2023intrinsicdimensionestimationrobust} suggest that zero-dimensional features can be particularly useful in text domains. To investigate this, we performed additional experiments using zero-dimensional cross-persistence densities. These densities were calculated by positioning all zero-dimensional features along the y-axis, see Fig.~\ref{fig:Texts_results_zero}. On the left side of the figure, we show the real cross-persistence densities, and on the right side, the predicted cross-persistence densities. Cross-RipsNet learned the specific characteristics of this data and demonstrated good prediction quality.

\begin{figure}[t]
\vspace{-.15in}
\centering
\subfloat[Train]
{\includegraphics[width=0.48\columnwidth]{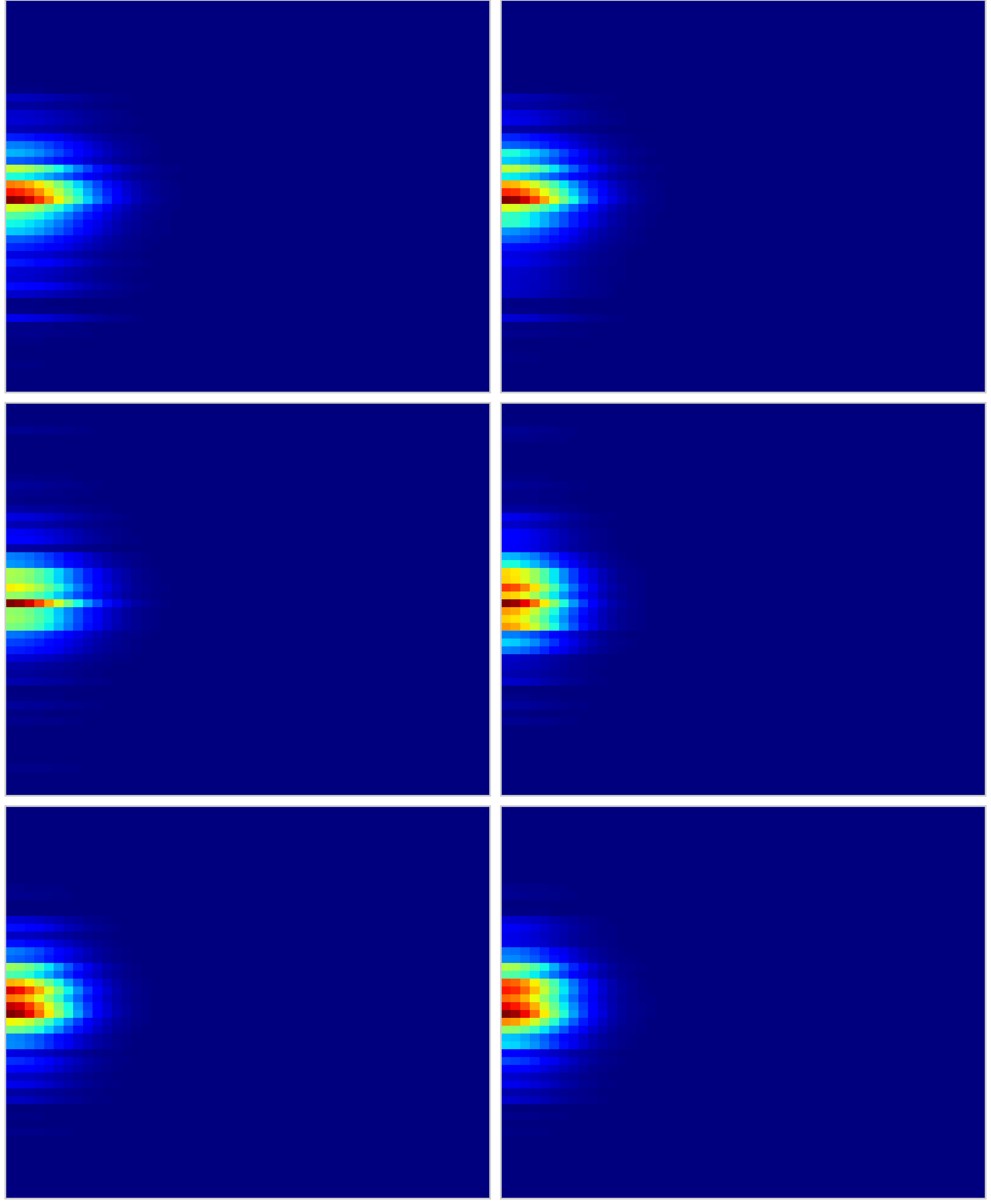}}
\hspace{0.02\columnwidth} 
\subfloat[Test]
{\includegraphics[width=0.48\columnwidth]{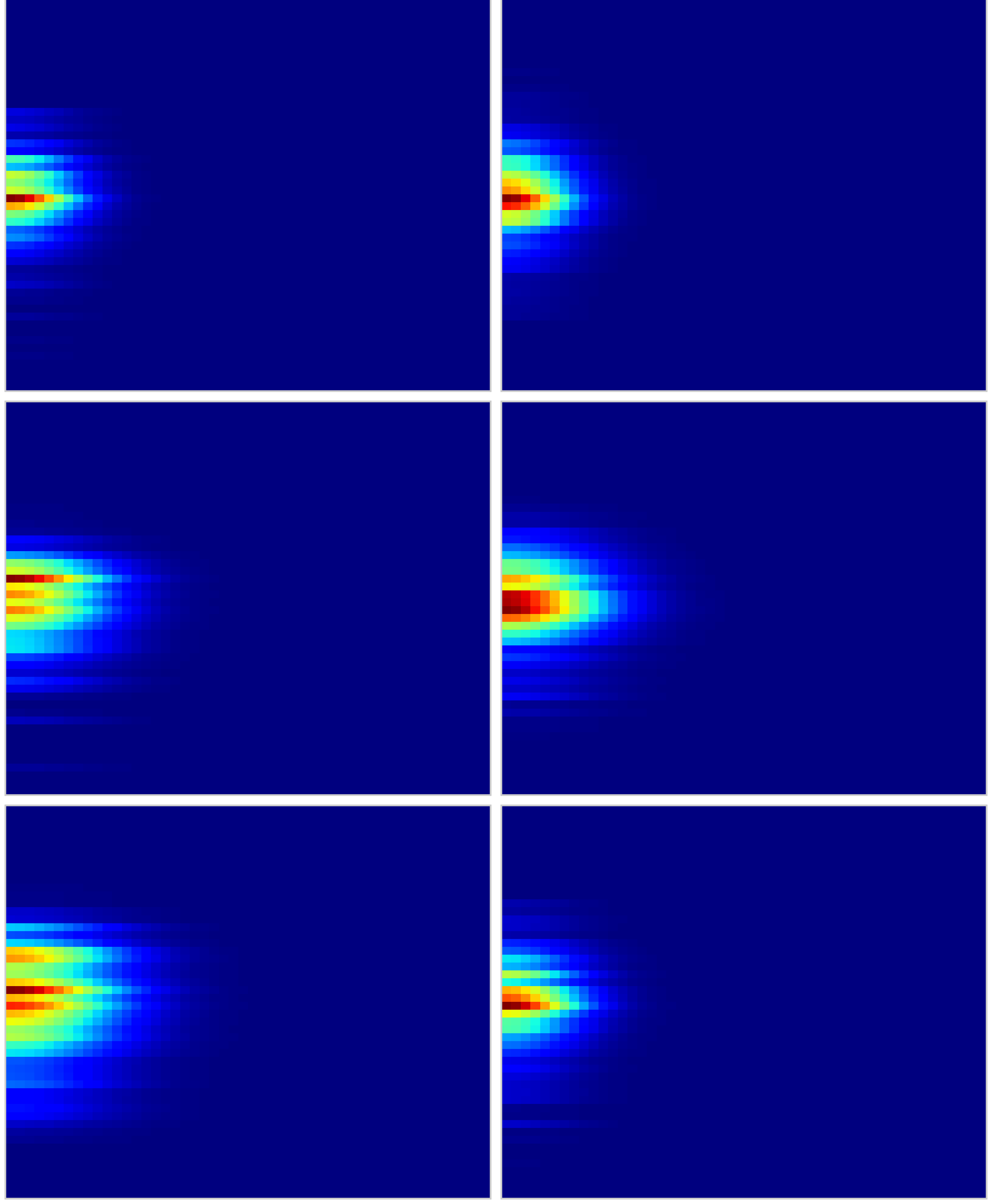}}
\caption{Results for text data for zero-dimensional homology case.  For each part of data, the left column shows the real persistence image, and the right column shows the predicted persistence image. Cross-RipsNet effectively learns the specificity of the data and achieves good results.}
\label{fig:Texts_results_zero}
\end{figure}

\subsection{Time estimation and complexity comparison}
 We measured time for calculation $20\%$ of the densities of expected Cross-persistence diagrams(cross-barcode) by straightforward classical algorithm and fitting Cross-RipsNet on $80\%$ of all data and then predict the rest of densities for three different domains. It appears that Cross-RipsNet can substantially faster do this task, see Table~\ref{tab:model_time}. Even more we skip diagrams calculation, we also skip PI calculation which may be very hard to calculate with bootstrapping original technique for estimating bandwidth parameter.

\rev{All timing experiments were conducted on a single workstation equipped with an NVIDIA RTX~3060 GPU and 32\,GB of RAM. For the classical pipeline, the reported times include the full preprocessing cost required to compute cross-persistence diagrams and their associated persistence images for the evaluated subset of data. For Cross-RipsNet, the reported times include all model-related preprocessing steps, including distance matrix computation, dimensionality reduction, as well as training on $80\%$ of the data and inference on the remaining $20\%$. Persistence images for the training split are assumed to be precomputed and loaded from disk in the latter case.}

\begin{table}[!t]
\caption{Time comparison of cross-barcode density computation for each domain using classical methods versus Cross-RipsNet (trained on $80\%$ of data, predicting $20\%$). Times are reported in hours.}
\label{tab:model_time}
\centering
\setlength{\tabcolsep}{6pt} 
\renewcommand{\arraystretch}{1.2} 
\footnotesize
\begin{tabular}{lccc}
\toprule
\textbf{Domain} & \textbf{Cross-Barcodes} & \textbf{Training + Predicting} & \textbf{Time Gain} \\
\midrule
3D & 2.60 & 0.40 & 6.5$\times$ \\
Text & 3.00 & 0.74 & 4.0$\times$ \\
Synthetic & 0.60 & 0.55 & 1.15$\times$ \\
\bottomrule
\end{tabular}
\end{table}

\section{Topological feature generator}
\label{sec:TopGen}
In this section, we conduct experiments using cross-persistence diagrams and their linear and statistical characteristics to generate features from time series. This approach was inspired by the use of persistence entropy to detect gravitational waves \footnote{\href{https://giotto-ai.github.io/gtda-docs/0.5.1/notebooks/gravitational_waves_detection.html}{https://giotto-ai.github.io/gtda-docs/grav-waves}}, as shown in Fig.~\ref{fig:Gravitational_waves}. In this example, the authors performed the following steps:
\begin{enumerate}
    \item Generated 200-dimensional time delay embeddings of each time series
    \item Used PCA to reduce the time delay embeddings to 3-dimensions
    \item Used the Vietoris-Rips construction to calculate persistence diagrams of $H_1$ and $H_0$ generators
    \item Extracted feature vectors using persistence entropy
    \item Trained a binary classifier on the topological features
\end{enumerate}

\begin{figure}[ht]
\centering
\subfloat[Gravitational wave with noise]{%
    \includegraphics[width=0.49\columnwidth]{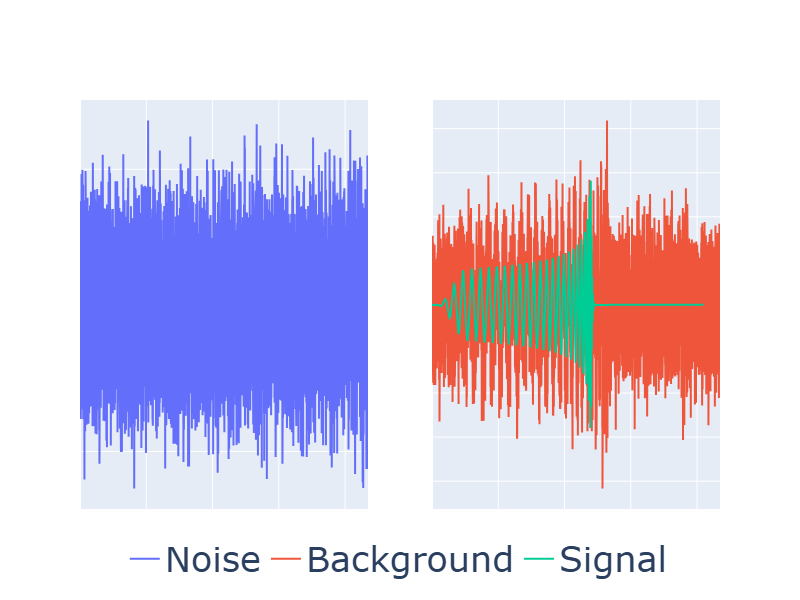}}
\subfloat[Time-delay embedded gravitational wave visualized in 3D]{%
    \includegraphics[width=0.49\columnwidth]{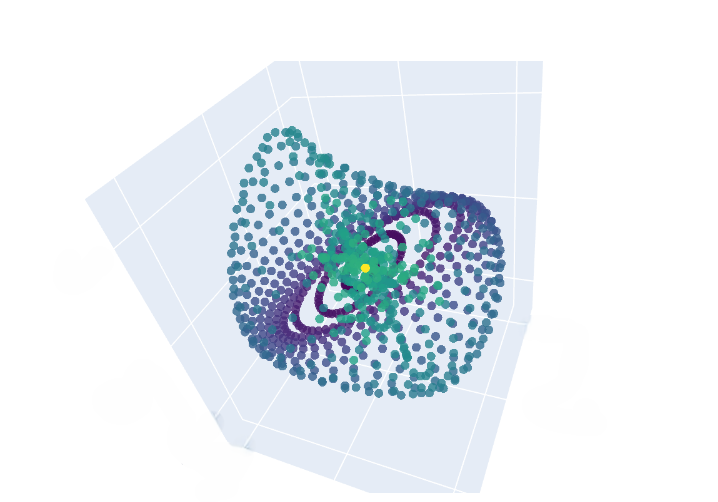}}
\caption{Example of gravitational wave noise process. (a) Gravitational wave with noise. (b) Time-delay embedded gravitational wave visualized in 3D.}
\label{fig:Gravitational_waves}
\end{figure}

After the second stage of the process, we can observe the cyclic structure of the obtained point clouds, which motivates \rev{the} use \rev{of} persistence homology, see Fig.~\ref{fig:Gravitational_waves}b.

Using this algorithm, the authors achieved good accuracy and a high ROC AUC score in this dataset.

\subsection{Gravitational waves}
In contrast to using only persistence entropy for $H_1$ and $H_0$ as in the example above, we can compute a much larger number of characteristics by using cross-persistence diagrams. To demonstrate the capabilities of this approach, we performed the same experiments.

To calculate a cross-persistence image, we need two point clouds. Consequently, our approach relies on comparing two time series. We selected one example from each class (gravitational wave, noisy gravitational wave, and pure noise) and used them as pairs for feature calculation for each time series.

We chose the following characteristics to calculate: MTD and cross-persistence entropy for each example, computed twice (once with the new time series on the left and once on the right side). In total, we obtain 12 values for each time series. As in the Giotto example, we used logistic regression for consistency with the baseline. In our experiments, we tested three different feature sets: only entropies, only MTD values, and both together. The results of the experiments are shown in Table~\ref{tab:Gravitational_exp}. By using cross-persistence diagrams with a single random example from each class for comparing topologies, our algorithm is able to extract valuable information about the topological structure of the given time series and classify them with good accuracy, which exceeds the accuracy of the baseline by $9\%$.


\begin{table*}[!t]
  \caption{Results of experiments with gravitational waves. Best scores are shown in \textbf{bold}}
  \label{tab:Gravitational_exp}
  \centering
  \setlength{\tabcolsep}{8pt} 
  \renewcommand{\arraystretch}{1.2} 
  \footnotesize
  \resizebox{\textwidth}{!}{ 
  \begin{tabular}{lcccc}
    \toprule
    \textbf{Model} & \textbf{Accuracy (train)} & \textbf{ROC (train)} & \textbf{Accuracy (valid)} & \textbf{ROC (valid)} \\
    \midrule
    \rev{Persistence Entropy (Giotto)} & $0.658\pm0.004$ & $0.703\pm0.005$ & $0.650\pm0.018$ & $0.700\pm0.059$ \\
    \rev{Cross-Entropies} & $0.666\pm0.001$ & $0.720\pm0.004$ & $0.700\pm0.030$ & $0.747\pm0.040$ \\
    MTDs & $0.686\pm0.007$ & $0.746\pm0.006$ & $0.704\pm0.054$ & $0.772\pm0.052$ \\
    MTDs+\rev{Cross-Entropies} & \textbf{0.692$\pm$0.004} & \textbf{0.776$\pm$0.005} & \textbf{0.711$\pm$0.040} & \textbf{0.783$\pm$0.040} \\
    \bottomrule
  \end{tabular}
  }
\end{table*}

\subsection{Time Series Classification}

To further evaluate our approach, we used six datasets from the UCR Time Series Classification Archive\footnote{https://www.timeseriesclassification.com/index.php} and developed a scikit-learn-compatible pipeline for topological feature generation, called \texttt{TopGen}. For comparison, we used two of the most popular methods for time series feature generation: CATCH22 and FreshPrince \cite{catch22, FreshPRINCE, bakeoff-revisited}. We used a Random Forest classifier for the downstream classification task, ensuring that the same parameters were applied for each feature generation method. We also investigated \texttt{TopGen} in multilabel classification tasks, in that case we took single random example from each classes for calculating features.

For consistency, we repeated the experiments with five different random seeds for each model. The accuracy scores and std. for the time series classification tasks are presented in Table~\ref{tab:Timeseries_exp}.

Although state-of-the-art (SOTA) models, such as FreshPrince and CATCH22, achieved higher accuracy compared to \texttt{TopGen}, the results suggest that \texttt{TopGen} can complement these methods by enhancing the feature sets.

    
    
    
    
    

\begin{table*}[!t]
  \caption{Experimental results on UCR datasets. Best scores are shown in \textbf{bold}; improved scores using TopGen features are \underline{underlined}.}
  \label{tab:Timeseries_exp}
  \centering
  \setlength{\tabcolsep}{7pt} 
  \renewcommand{\arraystretch}{1.2} 
  \footnotesize
  \resizebox{\textwidth}{!}{ 
  \begin{tabular}{l|ccccc}
    \toprule
    \textbf{Model} & \textbf{Worms} & \textbf{WormsTwoClass} & \textbf{Computers} & \textbf{Earthquakes} & \textbf{RefrigerationDevices} \\
    \midrule
    catch22 & $0.766\pm0.011$ & $0.784\pm0.019$ & $0.764\pm0.032$ & $0.804\pm0.012$ & $0.639\pm0.019$ \\
    TopGen & $0.701\pm0.010$ & $0.728\pm0.017$ & $0.657\pm0.017$ & $0.816\pm0.019$ & $0.494\pm0.016$ \\
    FreshPrince & $0.779\pm0.027$ & $0.777\pm0.036$ & \textbf{0.833$\pm$0.024} & \textbf{0.821$\pm$0.020} & \textbf{0.692$\pm$0.020} \\
    catch22+TopGen & \underline{0.779$\pm$0.024} & \textbf{0.784$\pm$0.028} & $0.758\pm0.030$ & \underline{0.814$\pm$0.020} & \underline{0.661$\pm$0.020} \\
    Fresh+TopGen & \textbf{0.783$\pm$0.026} & \underline{0.784$\pm$0.022} & $0.825\pm0.027$ & $0.820\pm0.021$ & $0.690\pm0.020$ \\
    \bottomrule
  \end{tabular}
  }
\end{table*}


In future work, we aim to further improve feature extraction by incorporating additional topological characteristics, developing more effective strategies for selecting examples from different classes, and using the density of cross-persistence diagrams instead of calculating individual cross-barcodes with random examples from each class. We also plan to test on a broader range of time series datasets. Additionally, we intend to explore more advanced classifiers and ensemble techniques to fully leverage the potential of topological features in time series analysis.

\section{Conclusion}
\label{sec:conclusion}

We have demonstrated that cross-persistence diagrams admit a density with respect to the Lebesgue measure, enabling the use of classical statistical tools in scenarios involving interactions between two manifolds. This theoretical result opens the door to applying statistical inference techniques in diverse applications, such as manifold discrimination, feature extraction for time series, and the detection of AI-generated content.

To support practical deployment, we introduced Cross-RipsNet, a neural architecture capable of efficiently estimating the density of cross-persistence diagrams and their linear representations at a significantly reduced computational cost compared to direct calculation. Our experiments on both synthetic and real-world datasets confirm that Cross-RipsNet effectively captures joint topological structures that are often overlooked by traditional methods. In addition, we explored the application of the proposed framework to time-series classification. While this setting presents additional challenges and remains an area for further development, the results indicate that cross-persistence features can provide complementary structural information in sequential data analysis.

Overall, our framework provides a flexible and computationally efficient approach for leveraging topological information in machine learning tasks. Future work may further enhance its performance, particularly in time-series classification and other settings involving high-dimensional or structured data.

\appendices

\section{\break The density of persistence diagrams}
\label{app:theorem_proof_base}

At this stage, we recall the key propositions regarding the density of persistence diagrams established in \cite{Density-Chazal}, as they will be essential to prove our theorem.

Since the construction of persistence diagrams relies on sequentially adding simplices to the filtration, all you need to know about this process is the time at which each simplex is added. Let \( n > 0 \) be an integer, and define \( \mathcal{F}_n \) as the collection of non-empty subsets of \( \{1, \ldots, n\} \). Consider a continuous function \( \varphi = (\varphi[J])_{J \in \mathcal{F}_n} : M^n \to \mathbb{R}^{\mathcal{F}_n} \), known as the \emph{filtering function}, which governs the order in which simplices enter the filtration. Specifically, a simplex \( J \) is included in the filtration at time \( \varphi[J] \). 

For a given point \( x = (x_1, \ldots, x_n) \in M^n \) and a simplex \( J \), we denote \( x(J) := (x_j)_{j \in J} \). The filtering function \( \varphi \) must satisfy the following assumptions:

\begin{itemize}
    \item[(K1)] \emph{Absence of interaction:} For $J \in \mathcal{F}_n, \varphi[J](x)$ only depends on $x(J)$.
    \item[(K2)] \emph{Invariance by permutation:} For $J \in \mathcal{F}_n$ and for $(x_1, \ldots, x_n) \in M^n$, if $\tau$ is a permutation of $\{1, \ldots, n\}$ whose support is included in $J$, then $\varphi[J](x_{\tau(1)}, \ldots, x_{\tau(n)}) = \varphi[J](x_1, \ldots, x_n)$.
    \item[(K3)] \emph{Monotony:} For $J \subset J' \in \mathcal{F}_n, \varphi[J] \leq \varphi[J']$.
    \item[(K4)] \emph{Compatibility:} For a simplex $J \in \mathcal{F}_n$ and for $j \in J$, if $\varphi[J](x_1, \ldots, x_n)$ is not a function of $x_j$ on some open set $U$ of $M^n$, then $\varphi[J] \equiv \varphi[J \setminus \{j\}]$ on $U$.
    \item[(K5)] \emph{Smoothness:} The function $\varphi$ is subanalytic and the gradient of each of its entries (which is defined a.s.e.) is non-vanishing a.s.e.
\end{itemize}

These assumptions ensure that a filtration can be properly defined only by \emph{filtering function}.

For a given \( x \in M^n \), the different values attained by \( \varphi(x) \) in the filtration can be ordered as \( r_1 < \dots < r_L \). We define \( E_l(x) \) as the set of simplices \( J \) for which \( \varphi[J](x) = r_l \). The sets \( E_1(x), \dots, E_L(x) \) form a partition of \( \mathcal{F}_n \), which we denote by \( A(x) \).

In original paper next lemmas and theorem were proven for $l\geq1$, but for our case the constraint $l>1$ is more suitable.

\begin{lemma}  \label{lm: minimal_el}
    For a.s.e. $x \in M^n$, for $l > 1$, $E_l(x)$ has a unique minimal element $J_l$ (for the partial order induced by inclusion).
\end{lemma}
\begin{lemma}  \label{lm: partitions}
    A.s.e., $x \rightarrow A(x)$ is locally constant and the space $M^n$ is partitioned into a negligible set $N(M^n)$ and some open subanalytic sets $U_1,\dots,U_r$ on which $A$ is constant.
\end{lemma}
\begin{lemma}  \label{lm: grad_ex}
 Fix $1\leq r \leq  R$ and assume that $J_1, \dots,J_L$ are the minimal elements of $A$ on $U_r$. Then, for $1< l \leq L$ and $j\in J_l$,$\nabla^j \varphi[J_l]\neq 0$ a.s.e. on $U_r$.
\end{lemma}

\begin{theorem} \label{th:pers_density}
    Fix $n \geq 1$. Assume that $M$ is a real analytic compact $d$-dimensional connected submanifold possibly with boundary and that $\mathbf{X}$ is a random variable on $M^n$ having a density with respect to the Hausdorff measure $\mathcal{H}_{dn}$. Assume that $\mathcal{K}$ satisfies the assumptions (K1)-(K5). Then, for $s \geq 0$, the expected measure $\mathbb{E}[D_s[\mathcal{K}(\mathbf{X})]]$ has a density with respect to the Lebesgue measure on $\Delta$.
\end{theorem}

One of the main instrument which was used by the author of original paper is classical result in geometric measure theory is called coarea formula.
\begin{theorem}[Coarea formula]\label{thm:coarea}
    Let \( M \) (resp. \( N \)) be a smooth Riemannian manifold of dimension \( m \) (resp. \( n \)). Assume that \( m \geq n \) and let \( \Phi: M \to N \) be a differentiable map. Denote by \( D\Phi \) the differential of \( \Phi \). The Jacobian of \( \Phi \) is defined by
    \begin{equation}
    J\Phi = \sqrt{\det((D\Phi) \times (D\Phi)^t)}.
    \end{equation}
    For \( f: M \to \mathbb{R}_+ \) a positive measurable function, the following equality holds:
    \begin{IEEEeqnarray}{lCr}
    \int_M f(x) J\Phi(x) d\mathcal{H}_m(x)= \notag \\
    = \int_N \left( \int_{x \in \Phi^{-1}(\{y\})} f(x) d\mathcal{H}_{m-n}(x) \right) d\mathcal{H}_n(y).
    \end{IEEEeqnarray}
\end{theorem}

\section{\break The density of cross-persistence diagrams}
\label{app:theorem_proof_cross}

In this section, we extend the result of Theorem~2.1 to show that cross-persistence diagrams admit a density under the given assumptions. Our approach follows the pipeline of \cite{Density-Chazal} but introduces a new filtration, referred to as the Vietoris-Rips cross-persistence filtration.

\medskip
\noindent
\textbf{Setup.}
We consider two real analytic, compact, $d$-dimensional connected submanifolds $M$ and $N$, each possibly with boundary. Let $\mathbf{X}$ be a random variable on $M^n$ and $\mathbf{Y}$ a random variable on $N^k$, both having densities with respect to the Hausdorff measures $\mathcal{H}_{dn}$ and $\mathcal{H}_{dk}$, respectively. Since we must handle both random variables simultaneously, we define
\begin{equation}
\mathbf{Z} \;=\; (\mathbf{X}, \mathbf{Y}).
\end{equation}
which also has a density on $M^n \times N^k$.

\medskip
\noindent
\textbf{Vietoris-Rips cross-persistence filtration.}
To construct the Vietoris-Rips cross-persistence filtration, we introduce a \emph{filtering function}
\begin{equation}
\varphi \;=\; \bigl(\varphi[J]\bigr)_{J \in \mathcal{F}_{n+k}}
:\;
M^n \times N^k \;\to\; \mathbb{R}^{\mathcal{F}_{n+k}}.
\end{equation}

This function is defined over the collection of simplices \( J \in \mathcal{F}_{n+k} \), where each simplex \( J \) corresponds to a subset of vertices from the two manifolds \( M \) and \( N \). Specifically, \( J_X \subseteq M^n \) and \( J_Y \subseteq N^k \) are subsets of vertices corresponding to the random variables \( \mathbf{X} \) and \( \mathbf{Y} \), respectively, and \( J = J_X \cup J_Y \).

The filtering function is then defined as
\begin{equation}
\varphi[J](\mathbf{X}, \mathbf{Y}) = 
\begin{cases}
\max\limits_{\substack{i \in J_X \cup J_Y \\ j \in J_X}} \|x_i - x_j\|, & \text{if } J_X \neq \emptyset, \\
0, & \text{otherwise.}
\end{cases}
\end{equation}

Here, $J_X$ and $J_Y$ denote the subsets of vertices in $J$ corresponding to $\mathbf{X}$ and $\mathbf{Y}$, respectively, and $J = J_X \cup J_Y$.  


\rev{We emphasize that the Vietoris-Rips cross-persistence filtering function does not satisfy assumption (K5) globally. Indeed, for any singleton or simplex $J$ such that $J_X=\emptyset$, we have $\varphi[J]\equiv 0$, and consequently $\nabla \varphi[J]$ vanishes everywhere. This situation is fully analogous to the setting considered in \cite[Theorem~3.3]{Density-Chazal}, where the authors treat the case of 0-persistence separately due to the same degeneracy. Following this strategy, we split the analysis into two parts. For $s\ge 1$, all persistence pairs necessarily involve at least one simplex with $J_X\neq\emptyset$, and on the corresponding open subanalytic sets the filtering functions are subanalytic and satisfy the required non-degeneracy and rank conditions a.s.e.. The case $s=0$, which consists of persistence pairs with birth time equal to zero and simplices supported entirely on $\mathbf{Y}$ , is handled separately, where the relevant mapping has full rank relative to its image and admits a density on the vertical line $\{0\}\times[0,\infty)$.}

\begin{theorem}\label{th: cross_pers_density}
Fix $n,k \ge 1$. Suppose $M$ and $N$ are real analytic, compact, $d$-dimensional connected submanifolds (possibly with boundaries), and let $\mathbf{X}$ and $\mathbf{Y}$ be random variables on $M^n$ and $N^k$, respectively, each having a density with respect to the Hausdorff measures $\mathcal{H}_{dn}$ and $\mathcal{H}_{dk}$. Then, for the $\mathcal{K}$-Vietoris-Rips cross-persistence filtration and any $s \ge 1$, the expected measure
\begin{equation}
\mathbb{E}\bigl[D_s[\mathcal{K}(\mathbf{X})]\bigr]
\end{equation}
admits a density with respect to the Lebesgue measure on $\Delta$. Furthermore,
\begin{equation}
\mathbb{E}\bigl[D_0[\mathcal{K}(\mathbf{X})]\bigr]
\end{equation}
has a density with respect to the Lebesgue measure on the vertical line $\{0\}\times[0,\infty)$.
\end{theorem}

\begin{proof}

\textbf{Partition of $M^n \times N^k$.}
We begin by partitioning $M^n \times N^k$ into finitely many open subanalytic sets $U_1,\ldots,U_R$, on each of which the partition $A$ (induced by the Vietoris-Rips cross-persistence filtration) is constant. Concretely, let
\begin{equation}
A \;=\; \bigl\{E_1,\dots,E_L\bigr\}
\end{equation}
be the partition of $\mathcal{F}_{n+k}$ determined by the filtration. Among these sets, $E_1(\mathbf{z})$ corresponds to $r_1=0$ and contains all singletons and simplices formed solely by points from $\mathbf{Y}$. Each set $E_l$ with $l>1$ has a unique minimal element $J_l$ (see Lemma~\ref{lm: minimal_el}). For $E_1$, we pick $J_1=\{1\}$ as the minimal element.

Using these minimal elements, one shows that the map $\mathbf{z}\mapsto A(\mathbf{z})$ is locally constant on each $U_r$, thus yielding a partition of $M^n\times N^k$ into a negligible set plus the open subanalytic sets $U_1,\ldots,U_R$ (see Lemma~\ref{lm: partitions}). If $J_1,\ldots,J_L$ are the minimal elements on $U_r$, then for each $l>1$ and $j\in J_l$, we have
\begin{equation}
\nabla^j \varphi[J_l] \;\neq\; 0
\quad
\text{a.s.e.\ on }U_r
\end{equation}
(see Lemma~\ref{lm: grad_ex}).  

\medskip
\noindent
\textbf{Defining $V_r$.}
We now define
\begin{equation}
V_r
\;=\;
U_r
\;\setminus\;
\Bigl(\,
\bigcup_{l=2}^L \bigcup_{j=1}^{\mid J_l\mid}
\{\nabla^j \varphi[J_l] = 0\}
\Bigr).
\end{equation}
On each set \( V_r \), the diagram \( D_s[\mathcal{K}(\mathbf{z})] \) can be decomposed as
\begin{equation}
D_s[\mathcal{K}(\mathbf{z})] \text{ on } V_r = \sum_{i=1}^{N_r} \delta_{r_i},
\end{equation}
where \( \delta_{r_i} \) represents the Dirac delta measure corresponding to the feature \( r_i \) in the diagram.

where $r_i = \bigl(\varphi[J_{l_1}](\mathbf{z}),\,\varphi[J_{l_2}](\mathbf{z})\bigr)$ for certain $l_1,l_2$.

\medskip
\noindent
\textbf{Rank of $\Phi_{ir}$.}
For $s\ge 1$, we wish to compute the expected diagram

\begin{IEEEeqnarray}{rCl}
    E\bigl[D_s[\mathcal{K}(\mathbf{X})]\bigr] &=& \sum_{r=1}^{R} E\Bigl[\mathbf{1}\{\mathbf{X} \in V_r\} \; D_s[\mathcal{K}(\mathbf{X})]\Bigr] \notag \\
    &=& \sum_{r=1}^{R} \sum_{i=1}^{N_r} E\Bigl[\mathbf{1}\{\mathbf{X} \in V_r\} \; \delta_{r_i}\Bigr].
\end{IEEEeqnarray}
Define
\begin{equation}
\Phi_{ir}: \quad
x \;\in\; V_r
\;\;\mapsto\;
r_i
\;=\;
\Bigl(\varphi[J_{l_1}](x),\;\varphi[J_{l_2}](x)\Bigr).
\end{equation}
Let $\kappa$ be the density of $\mathbf{Z}$ with respect to $ \mathcal{H}_{dn}\times\mathcal{H}_{dk}$.  
\rev{We apply the Coarea Formula (Theorem~\ref{thm:coarea}), noting that for $s\ge1$ the mapping $\Phi_{ir}$ has full rank (equal to 2) on $V_r$ a.s.e.} Indeed:

\begin{itemize}
\item If $j \in (J_{l_2}\setminus J_{l_1})$, then $\nabla^j\varphi[J_{l_2}](x)\neq 0$, while $\nabla^j\varphi[J_{l_1}](x)=0$.
\item If $j' \in J_{l_1}$, then $\nabla^{j'}\varphi[J_{l_1}](x)\neq 0$.
\end{itemize}
Hence,

\begin{IEEEeqnarray}{lcr}
    E\Bigl[\mathbf{1}\{\mathbf{X} \in V_r\} \, \delta_{r_i}\Bigr] = P\bigl(\Phi_{ir}(\mathbf{X}) \in B, \; \mathbf{X} \in V_r\bigr) \notag \\
    = \int_{V_r} \mathbf{1}\{\Phi_{ir}(x) \in B\} \, \kappa(x) \, d\mathcal{H}_{nd}(x) \notag \\
    = \int\limits_{u \in B} \int_{x \in \Phi_{ir}^{-1}(\{u\})} \bigl(J\Phi_{ir}(x)\bigr)^{-1} \, \kappa(x) \, d\mathcal{H}_{nd-2}(x) \, du.
\end{IEEEeqnarray}

This shows $E\bigl[\mathbf{1}\{\mathbf{X}\in V_r\}\,\delta_{r_i}\bigr]$ admits a density with respect to the Lebesgue measure on $\Delta$. Consequently, the function
\begin{equation}
p(u)
\;=\;
\sum\limits_{r=1}^{R}
\sum\limits_{i=1}^{N_r}
\int_{x\in \Phi_{ir}^{-1}(\{u\})}
\bigl(J\Phi_{ir}(x)\bigr)^{-1}\,
\kappa(x)\,
d\mathcal{H}_{nd-2}(x).
\end{equation}
constitutes the density of $E[D_s[\mathcal{K}(\mathbf{X})]]$ on $\Delta$.

\medskip
\noindent
\textbf{Case $s=0$.}
For 0-persistence, the path is similar but requires additional attention to $E_1$, which includes all singletons and simplices consisting solely of points from $\mathbf{Y}$, appended at $t=0$. Consequently, for every $J\in E_1$, $\varphi[J]\equiv 0$. In the pair $\bigl(\varphi[J_{l_1}](x),\,\varphi[J_{l_2}](x)\bigr)$, the component $\varphi[J_{l_1}](x)$ is always zero, while $\nabla^j\varphi[J_{l_2}](x)\neq 0$ for $j\in J_{l_2}$. Thus, the image of $\Phi_{ir}$ lies in the set $\{0\}\times [0,\infty)$, and $\Phi_{ir}$ has rank 1, which is full rank relative to its image. The same application of the Coarea Formula therefore establishes that
\begin{equation}
\mathbb{E}\bigl[D_0[\mathcal{K}(\mathbf{X})]\bigr]
\end{equation}
also admits a density with respect to the Lebesgue measure on $\{0\}\times [0,\infty)$.

\end{proof}

\section{\break The Density-Overlap Metric Discussion}
\label{app:overlap_metric}

\rev{We use the density-overlap functional:}

\begin{equation}\mathcal{O}(p,q)
=
\int_{\mathbb{R}}
\min\{p(z), q(z)\}\, dz,\end{equation}
\rev{as a continuous similarity score between estimated MTD distributions. Here, $p(z)$ and $q(z)$ denote one-dimensional probability density functions of the scalar random variable $Z=\mathrm{MTD}(Q_i,Q_j)$, obtained via a linear representation of cross-persistence diagrams as established in Proposition~\ref{th:linear_repr}. Fig.~\ref{fig:overlap_four_cases} illustrates four typical scenarios: full overlap, partial overlap, overlap with substantially different dispersion, and near-zero overlap.}
\begin{figure}[t]
\centering
\IfFileExists{pictures/overlap_four_cases.png}{%
    \includegraphics[width=0.49\textwidth]{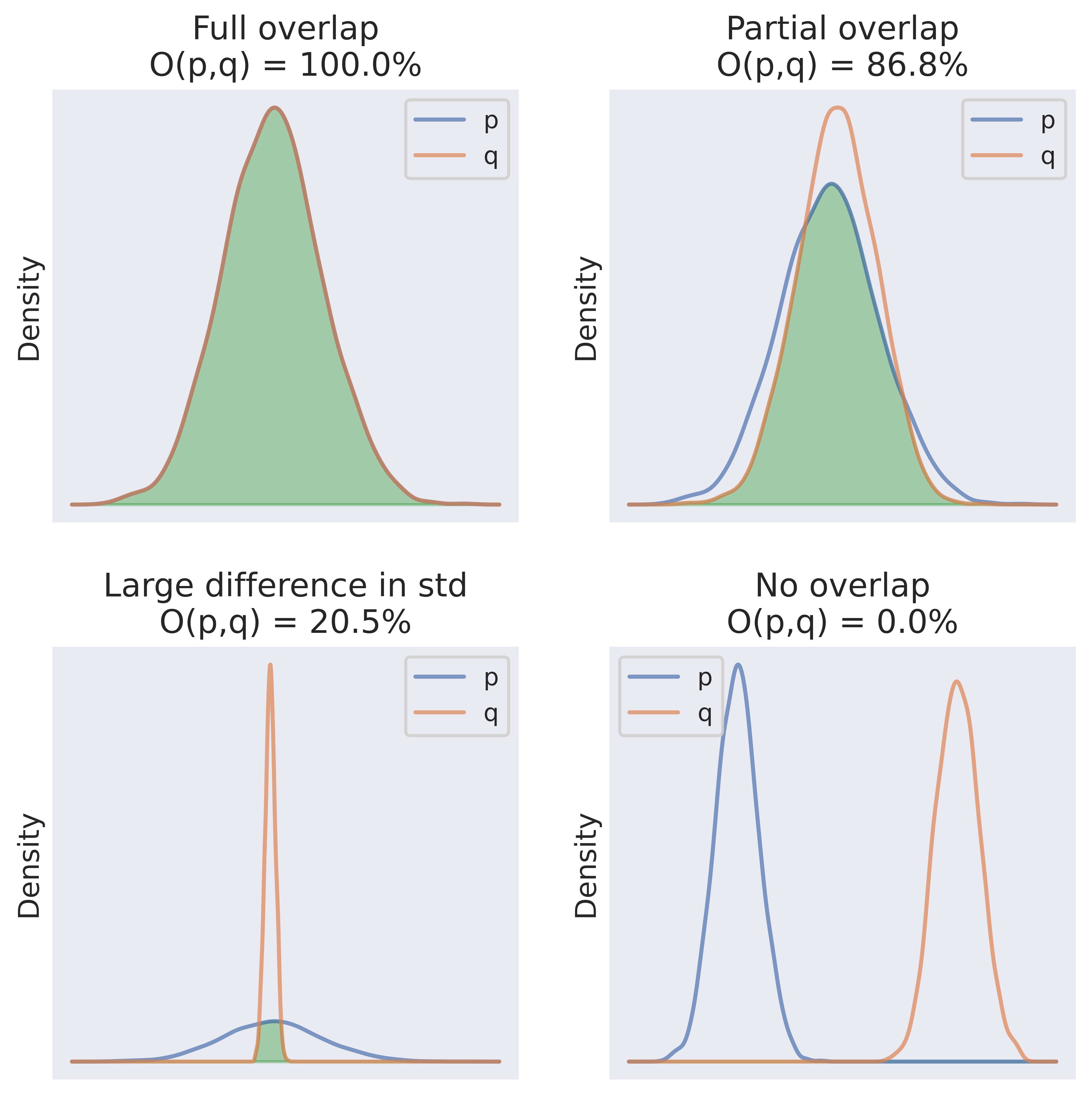}%
}{%
    \fbox{\parbox{0.95\textwidth}{\centering\rev{Place the file \texttt{pictures/overlap\_four\_cases.png} here.}}}%
}
\caption{\rev{Illustration of the overlap metric $\mathcal{O}(p,q)$ on four representative pairs of densities $p$ and $q$.}}
\label{fig:overlap_four_cases}
\end{figure}

\rev{
The density-overlap functional $\mathcal{O}(p,q)$ is an $L^1$-based similarity measure and is therefore insensitive to second-order geometric properties of the distributions, such as dispersion within the region of overlap. Consequently, two densities may exhibit moderate or even substantial overlap while differing significantly in their concentration or spread, as illustrated in the third scenario of Fig.~\ref{fig:overlap_four_cases}.
This regime is encountered in practice in high-complexity datasets such as CIFAR100 (Section~\ref{sec:experiments}), where overlap values remain well below saturation while dispersion differs substantially across class-conditional MTD densities.
}

\section{Stability of MTD Density and Overlap-Based Comparison}
\label{app:mtd_stability}

\rev{
In this appendix, we briefly justify the stability of the proposed MTD-based statistical comparison procedure.
}

\rev{
It is known that expected persistence diagrams admit densities and that kernel-based estimators constructed from independent realizations are statistically consistent under mild regularity assumptions. Such results follow from standard kernel density estimation theory for random measures and were established for persistence diagrams in~\cite{Density-Chazal}.
}

\rev{
Moreover, for suitable choices of bandwidth (e.g., selected via cross-validation), convergence rates of the corresponding estimators can be derived by adapting classical analyses of kernel density estimators to the persistence diagram setting, see~\cite{Density-Chazal,Tsybakov2008IntroductionTN}.
}

\rev{
Since cross-persistence diagrams satisfy the same structural assumptions as classical persistence diagrams (Theorem~\ref{th: cross_pers_density}), these consistency and convergence properties extend directly to the cross-persistence setting.
}

We now consider the Manifold Topology Divergence (MTD), which is a linear functional of a cross-persistence diagram defined by
\begin{equation}
\mathrm{MTD}_i(P, Q)
:=
\sum_{(\alpha_b, \alpha_d)\in \mathrm{CrossBarcode}_i(P,Q)} (\alpha_d-\alpha_b).
\end{equation}

In practical settings, due to finite sample size and bounded homological complexity, we assume that a cross-persistence diagram contains at most $N$ off-diagonal points almost surely. This ensures that the MTD functional is well-defined and integrable.

\begin{proposition}
Let $D$ be a random cross-persistence diagram and let
\begin{equation}
MTD(D) := \sum_{(x,y)\in D} (y-x)
\end{equation}
denote the MTD functional on diagram. Let $\mu$ and $\hat\mu$ be two probability laws on the space of cross-persistence diagrams such that
\begin{equation}
\|\mu - \hat\mu\|_{\mathrm{TV}} \le \varepsilon.
\end{equation}
Then the induced distributions of the random variable $MTD(D)$ satisfy
\begin{equation}
\|\rho_{\mathrm{MTD}} - \hat\rho_{\mathrm{MTD}}\|_1 \le \varepsilon.
\end{equation}
\end{proposition}

\begin{proof}
The $MTD$ functional is measurable. The distributions $\rho_{\mathrm{MTD}}$ and $\hat\rho_{\mathrm{MTD}}$ are the pushforward measures of $\mu$ and $\hat\mu$ under $MTD$.  
It is a standard property of total variation distance that pushforward by a measurable mapping does not increase total variation. Hence,
\begin{equation}
\|\rho_{\mathrm{MTD}} - \hat\rho_{\mathrm{MTD}}\|_{\mathrm{TV}}
\le
\|\mu - \hat\mu\|_{\mathrm{TV}}
\le \varepsilon.
\end{equation}
Since total variation coincides with the $L^1$ distance between densities when they exist, the claim follows.
\end{proof}

\begin{proposition}
Let $p,q$ and $\hat p, \hat q$ be true and estimated densities of two MTD random variables, satisfying
\begin{equation}
\|p-\hat p\|_1 \le \delta, \qquad \|q-\hat q\|_1 \le \delta.
\end{equation}
Then the density-overlap functional
\begin{equation}
\mathcal{O}(p,q) = \int_{\mathbb{R}} \min\{p(z),q(z)\}\,dz
\end{equation}
satisfies
\begin{equation}
|\mathcal{O}(p,q) - \mathcal{O}(\hat p,\hat q)| \le 2\delta.
\end{equation}
\end{proposition}
\begin{proof}
The claim follows directly from the Lipschitz continuity of the $\min$ operator with respect to the $L^1$ norm.

For all $z\in\mathbb R$,
\begin{IEEEeqnarray}{lcr}
|\min\{p(z),q(z)\}-\min\{\hat p(z),\hat q(z)\}| \le \notag \\
\le |p(z)-\hat p(z)|+|q(z)-\hat q(z)|.
\end{IEEEeqnarray}
Integrating both sides yields
\begin{equation}
|\mathcal{O}(p,q)-\mathcal{O}(\hat p,\hat q)| \le  \|p-\hat p\|_1+\|q-\hat q\|_1 \le 2\delta.
\end{equation}
\end{proof}

\rev{As a consequence, the overlap-based comparison used in Section~\ref{sec:experiments} is stable under density estimation error. In particular, when the estimated overlap is significantly separated from the chosen significance threshold (e.g., $0.05$), the probability of misclassification due to density estimation error becomes negligible for sufficiently large sample size.}

\section{\break Comparative Analysis of Human and AI-Generated Texts}
\label{app:Human_GPT}

To conduct a more comprehensive analysis of texts, which represent the most complex domain among those studied, we focused on examining the densities of cross-persistence diagrams derived from comparing GPT-generated texts with human-written texts.

As illustrated in Fig.~\ref{fig:Human_GPT_densities}, we analyzed four types of densities (Human vs. GPT, Human vs. Human, GPT vs. Human, and GPT vs. GPT). Through a detailed visual inspection of these diagrams, we observed that the densities with human texts on the left side of the $Cross-Barcode(P, Q)$ exhibit  significantly greater visual diversity compared to those with GPT-generated texts on the left side. 

\begin{figure*}[!t]
\vspace{-.15in}
\centering
{\includegraphics[width=0.32\linewidth]{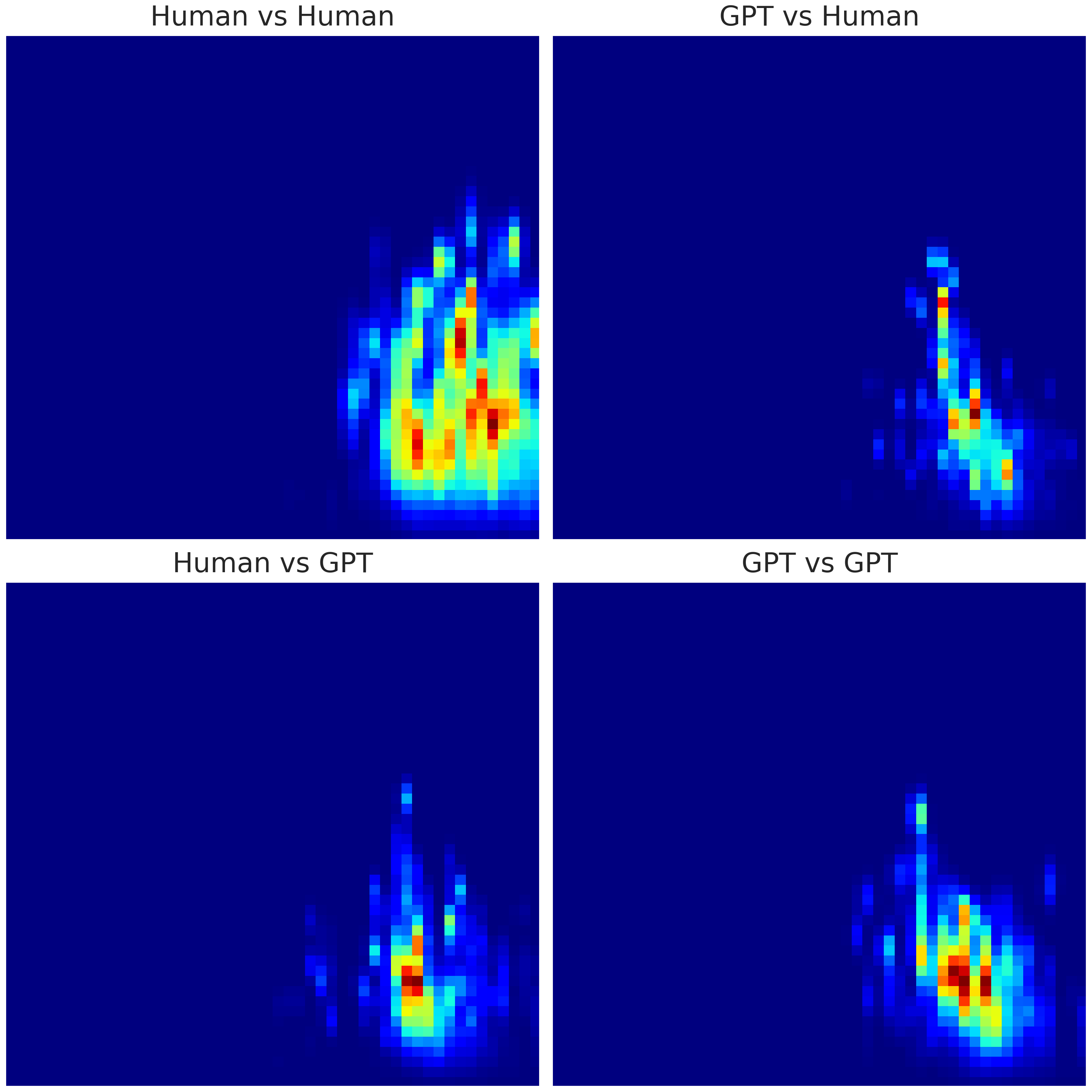}}
{\includegraphics[width=0.32\linewidth]{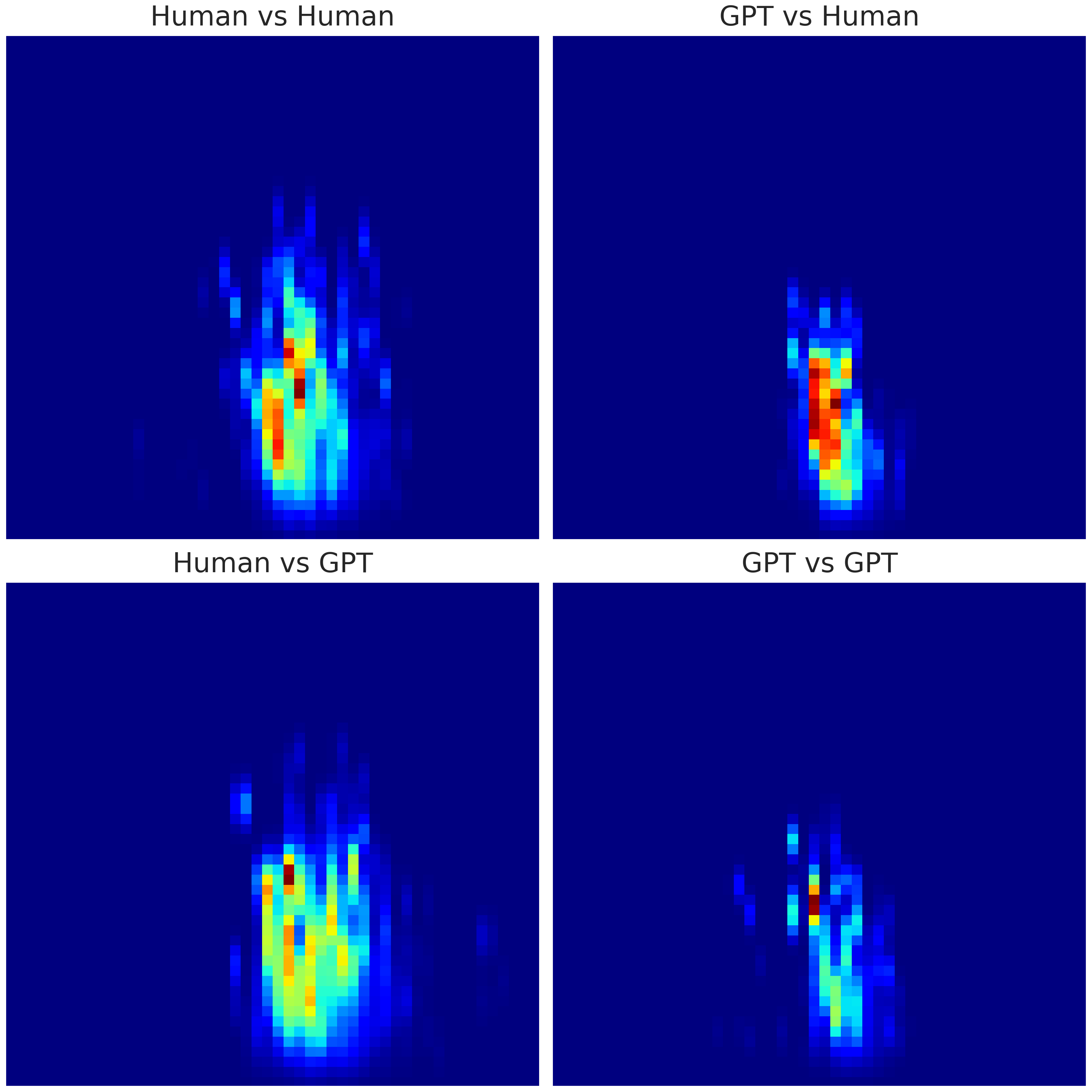}}
{\includegraphics[width=0.32\linewidth]{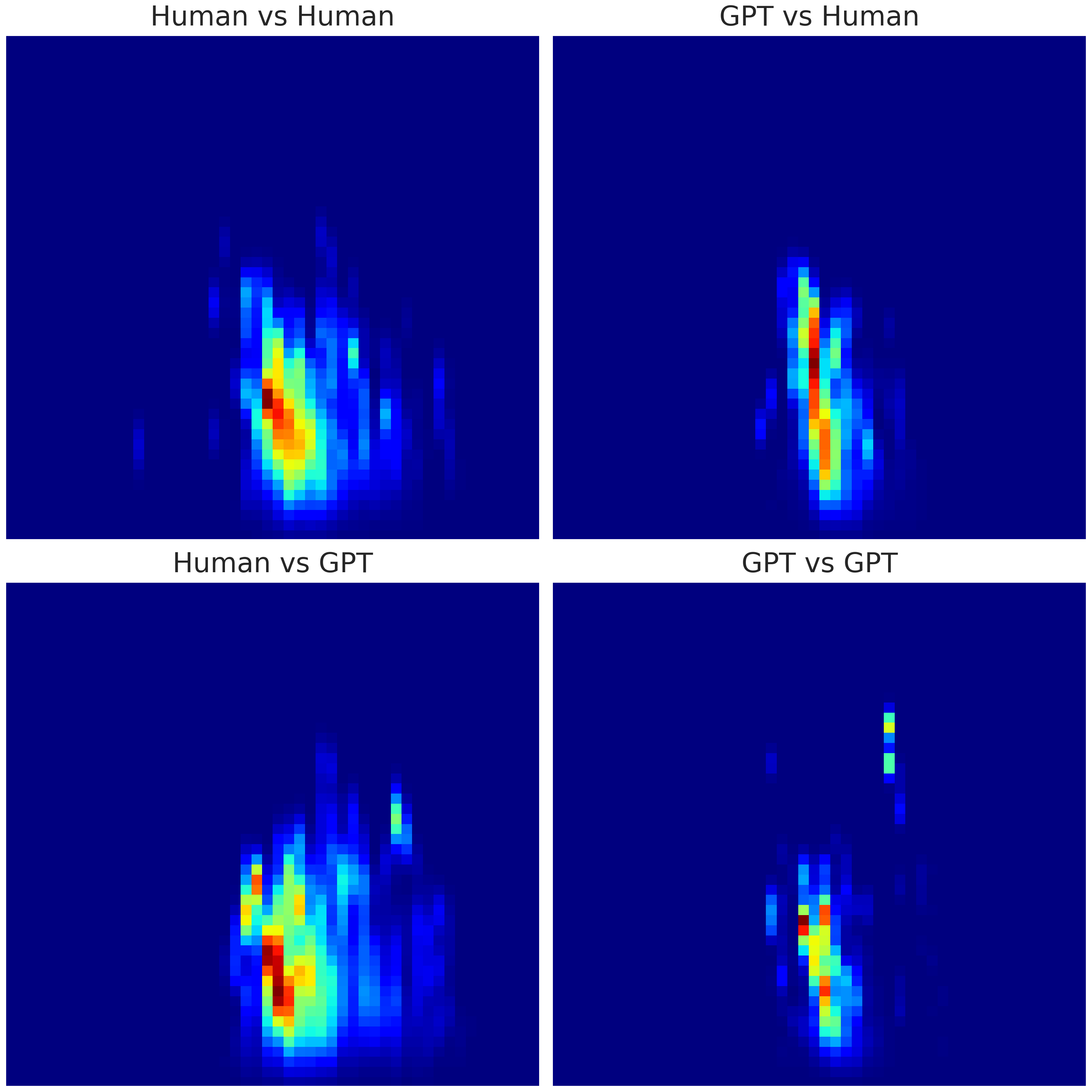}}
\caption{Comparison of cross-persistence diagram densities for different input combinations. A higher level of diversity is evident in the densities corresponding to human texts on the left side.}
\label{fig:Human_GPT_densities}
\end{figure*}

These observations motivated us to investigate the differences in diversity by applying dimensionality reduction techniques. We vectorized the densities of the cross-persistence diagrams and applied principal component analysis (PCA) and Isomap to project them onto a 2D space. The results, presented in Fig.~\ref{fig:Human_GPT_densities_projected}, demonstrate that these densities can be effectively separated into two distinct groups in the lower-dimensional representation.

To further assess the quality of the proposed features and their ability to differentiate whether the left point cloud in a pair is AI-generated or written by a human, we conducted an experiment with a baseline: calculating two persistence images for each point cloud, concatenating them, and applying PCA. As shown in Fig.~\ref{fig:Human_GPT_densities_projected_baseline}, our new approach can significantly better detect this difference and localize Human-Any pairs more effectively. 

\rev{These qualitative observations suggest that cross-persistence diagram densities capture structural differences between human-written and AI-generated texts. To substantiate this claim and to align the text-domain experiments with the quantitative protocol used in the gravitational wave study, we next introduce an explicit classification setting and report standard performance metrics.}

\begin{figure}[t]
\vspace{-.15in}
\centering
{\includegraphics[width=0.45\linewidth]{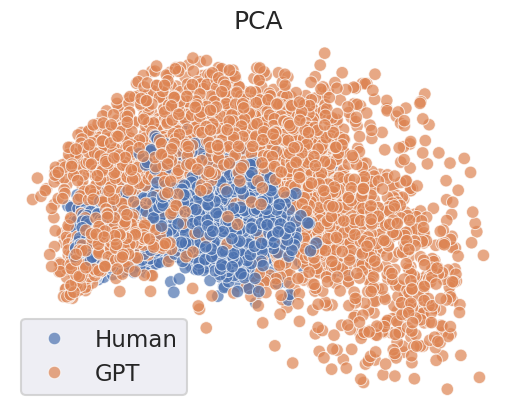}}
{\includegraphics[width=0.45\linewidth]{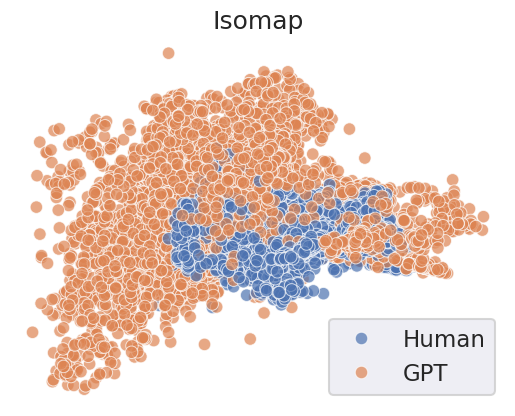}}
\caption{Projections of cross-persistence densities onto a 2D space using different dimensionality reduction techniques. Human vs Any text densities (blue) and AI vs Any text densities (orange) exhibit partial separation, confirming their distinct characteristics.}
\label{fig:Human_GPT_densities_projected}
\end{figure}

\begin{figure}[t]
\vspace{-.15in}
\centering
{\includegraphics[width=0.45\linewidth]{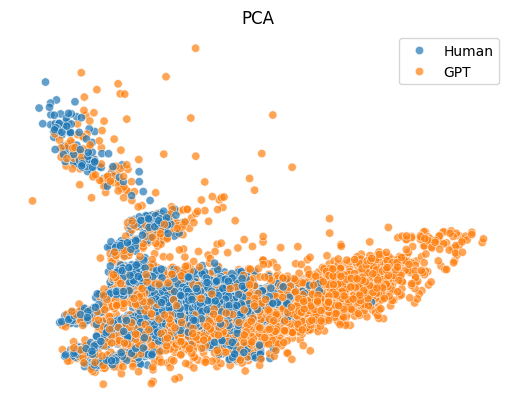}}
{\includegraphics[width=0.45\linewidth]{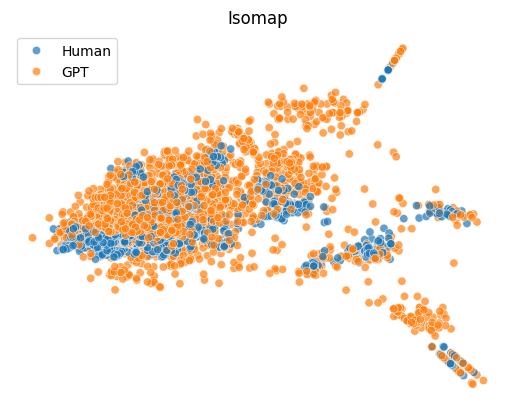}}
\caption{Projections of concatenated persistence images onto a 2D space using different dimensionality reduction techniques. It can be seen, that with usage of common persistence images, detection wether left cloud originates from AI or Human became much more harder, as support of Human points in this case are not localized.}
\label{fig:Human_GPT_densities_projected_baseline}
\end{figure}

\rev{The results in Figs.~\ref{fig:Human_GPT_densities}--\ref{fig:Human_GPT_densities_projected_baseline} reveal a systematic difference between cross-barcodes in which the \emph{left} point cloud corresponds to a human-written text and those in which it corresponds to an AI-generated text. This observation naturally motivates a supervised classification setting. Given a collection of reference texts from both classes, we compute cross-persistence diagrams between a new text and each reference text and extract feature vectors from the resulting densities. These representations are treated as samples drawn from class-dependent distributions. A standard classifier can then be trained to determine whether the left point cloud in a pair originates from a human or an AI-generated text.}

\rev{In this work, we adopt the same downstream evaluation protocol as in the gravitational wave experiment and train a logistic regression classifier on the proposed topological features. The corresponding quantitative results are reported in Table~\ref{tab:AI_Human_text_exp}, enabling a direct comparison with baseline representations based on classical persistence entropy of the query point cloud. The proposed cross-persistence based features consistently outperform the baseline persistence entropy representation in terms of both accuracy and ROC on the validation set. In particular, MTD-based features achieve the highest discriminative performance, confirming that cross-persistence densities capture information that is not accessible from single-cloud summaries alone. This confirms that incorporating cross-persistence information leads to a measurable improvement over single-cloud topological summaries in the text domain.}

\begin{table*}[!t]
  \caption{
  \rev{Results of AI vs. human text classification on Wiki and Reddit datasets.
  We report mean $\pm$ standard deviation over repeated runs.
  Wiki texts are generated using Wikipedia prompts, while Reddit texts are generated from question-based prompts following the experimental protocol of \cite{tulchinskii2023intrinsicdimensionestimationrobust}.
  Best validation scores within each dataset block are shown in \textbf{bold}.}
  }
  \label{tab:AI_Human_text_exp}
  \centering
  \setlength{\tabcolsep}{8pt}
  \renewcommand{\arraystretch}{1.2}
  \footnotesize
  \resizebox{\textwidth}{!}{
  \begin{tabular}{lcccc}
    \toprule
    \textbf{Model} & \textbf{Accuracy (train)} & \textbf{ROC (train)} & \textbf{Accuracy (valid)} & \textbf{ROC (valid)} \\
    \midrule
    
    \multicolumn{5}{c}{\textbf{Wiki dataset (Wikipedia prompts)}} \\
    \midrule
    Persistence Entropy (baseline) 
        & $0.839\pm0.003$ & $0.942\pm0.001$ 
        & $0.841\pm0.027$ & $0.945\pm0.007$ \\
    Cross-Entropies 
        & $0.948\pm0.004$ & $0.970\pm0.001$ 
        & $0.950\pm0.009$ & $0.977\pm0.007$ \\
    MTDs 
        & $0.955\pm0.001$ & $0.975\pm0.001$ 
        & $0.960\pm0.007$ & $0.980\pm0.006$ \\
    MTDs + Cross-Entropies 
        & 0.958$\pm$0.002 & 0.976$\pm$0.001
        & \textbf{0.966$\pm$0.007} & \textbf{0.981$\pm$0.006} \\

    \midrule
    \multicolumn{5}{c}{\textbf{Reddit dataset (question-based prompts)}} \\
    \midrule
    Persistence Entropy (baseline) 
        & $0.885\pm0.002$ & $0.944\pm0.002$ 
        & $0.873\pm0.027$ & $0.942\pm0.023$ \\
    Cross-Entropies 
        & $0.904\pm0.002$ & $0.942\pm0.002$ 
        & $0.912\pm0.010$ & $0.946\pm0.017$ \\
    MTDs 
        & $0.957\pm0.002$ & $0.988\pm0.001$ 
        & $0.957\pm0.009$ & $0.989\pm0.006$ \\
    MTDs + Cross-Entropies 
        & 0.965$\pm$0.002 & 0.994$\pm$0.000
        & \textbf{0.971$\pm$0.008} & \textbf{0.994$\pm$0.003} \\

    \bottomrule
  \end{tabular}
  }
\end{table*}

\section{\break Cross-RipsNet for predicting MTD density}
\label{app:MTD_density_prediction}

In this section, we conducted experiments with Cross-RipsNet to predict not merely the density of cross-persistence diagrams but specifically the density of Manifold Topology Divergence (MTD) directly. This investigation is particularly relevant, as certain applications may require solely MTD density prediction, similar to the use case presented in Section~\ref{sec:experiments}.

To evaluate Cross-RipsNet's capability in predicting such objects, we devised an experiment where the model was tasked with forecasting the MTD density for novel pairs of point clouds from a synthetic dataset. The results, as illustrated in Fig.~\ref{fig:predicting_MTD}, demonstrate that our model achieves this prediction with high accuracy.


\begin{figure}[t]
\vspace{-.15in}
\centering
{\includegraphics[width=0.49\linewidth]{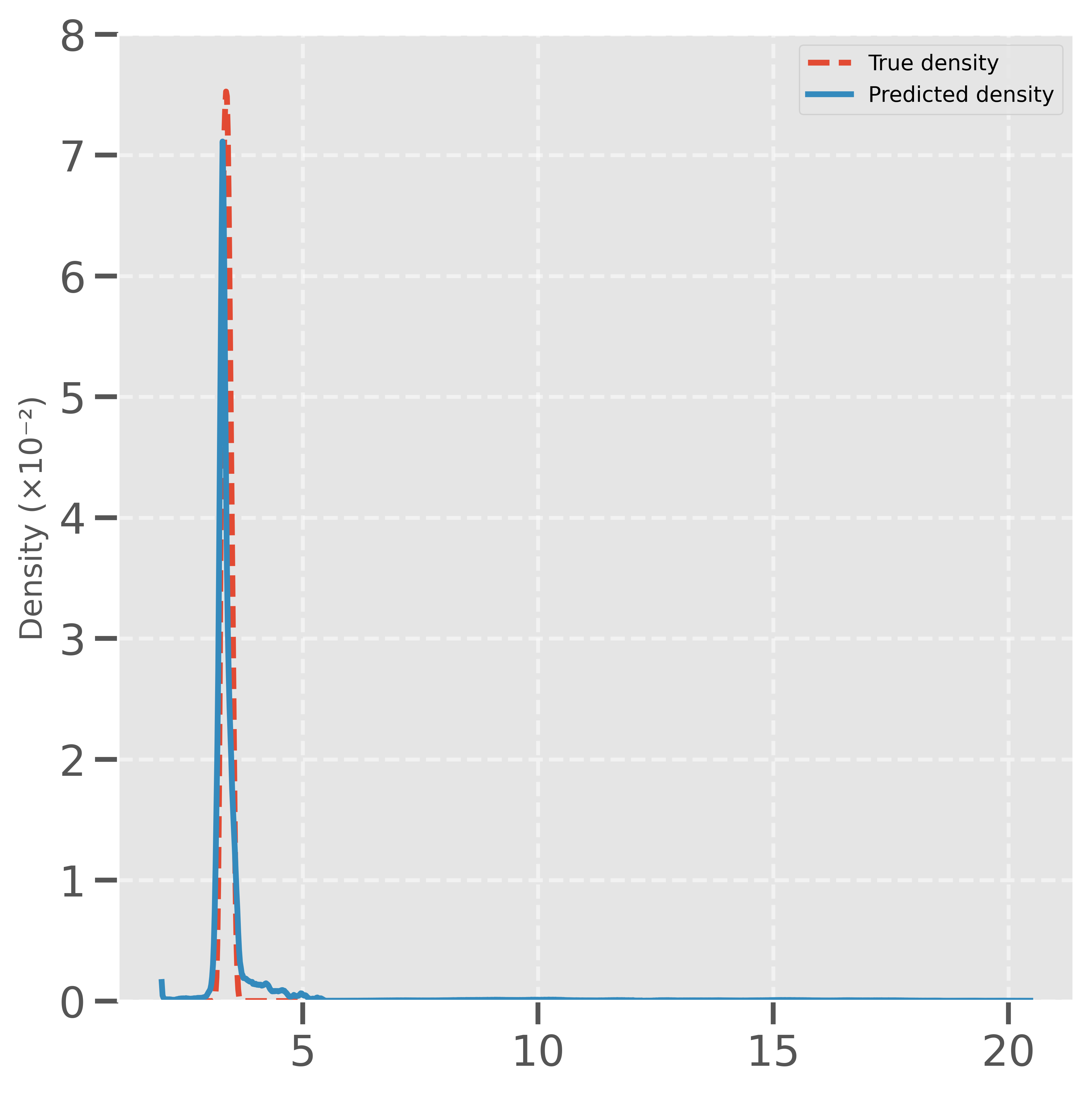}}
{\includegraphics[width=0.49\linewidth]{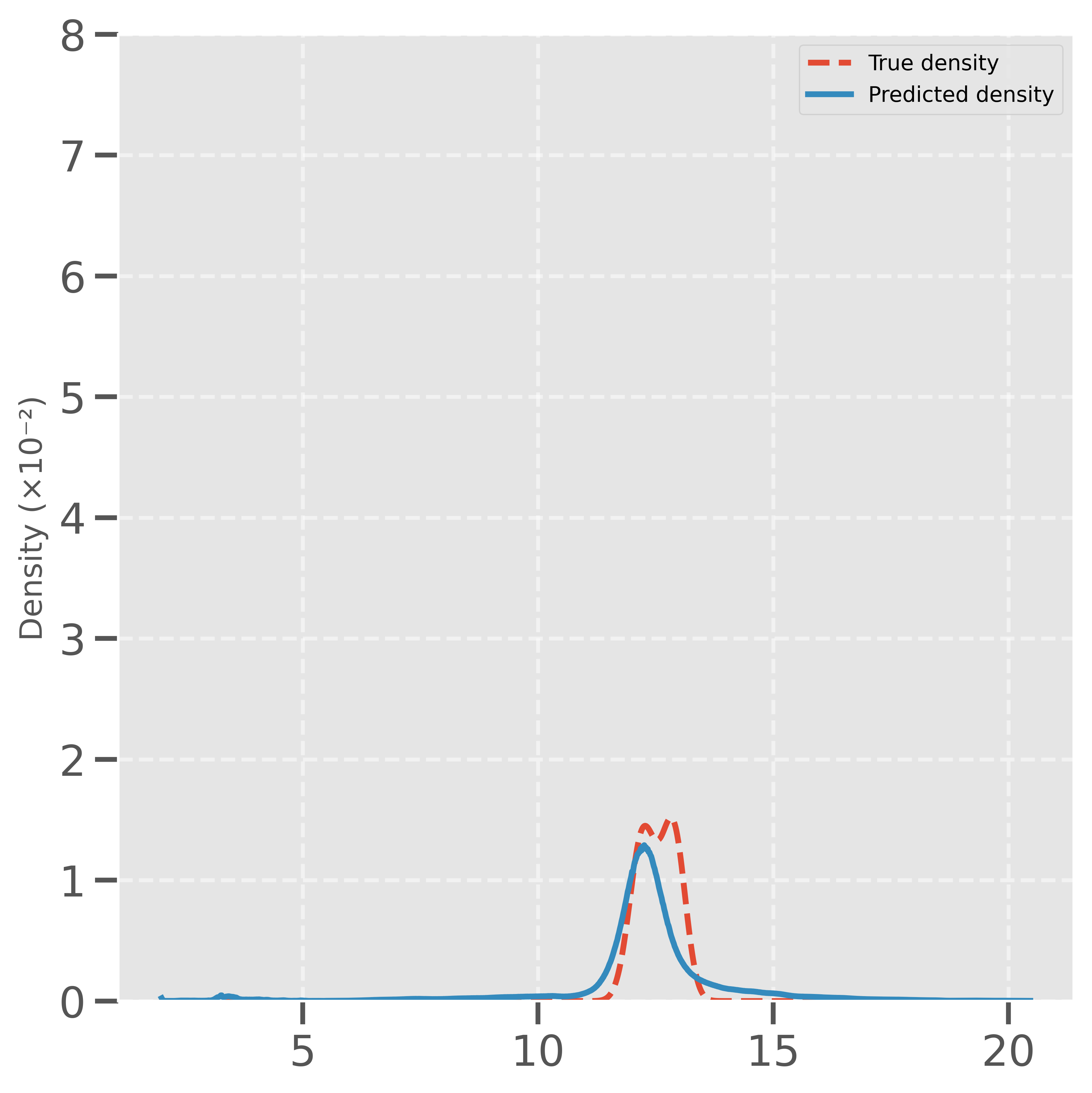}}
\caption{Cross-RipsNet's performance in predicting MTD density, for the synthetic dataset and two randomly chosen pairs of clouds from the test set.}
\label{fig:predicting_MTD}
\end{figure}

\section{\break Sensitivity Analysis of Noise Injection}
\label{app:noise_sensitivity}

\rev{Noise injection plays a key role in our empirical pipeline for improving the separability of image classes (Section ~\ref{sec:experiments}). To make the choice of noise magnitude transparent and to quantify its effect, we perform a sensitivity analysis. We measure the mean overlap (intersection) between estimated cross-persistence densities across classes as a function of the noise level. We compare two regimes: (i) noise is injected only into the right point cloud and (ii) noise is injected into both point clouds; the results  on the COIL20 dataset are summarized in Fig.~\ref{fig:coil20_noise_sensitivity}.}
\begin{figure}[t]
\centering
\IfFileExists{pictures/COIL20_noise_sensitivity.png}{%
    \includegraphics[width=0.95\linewidth]{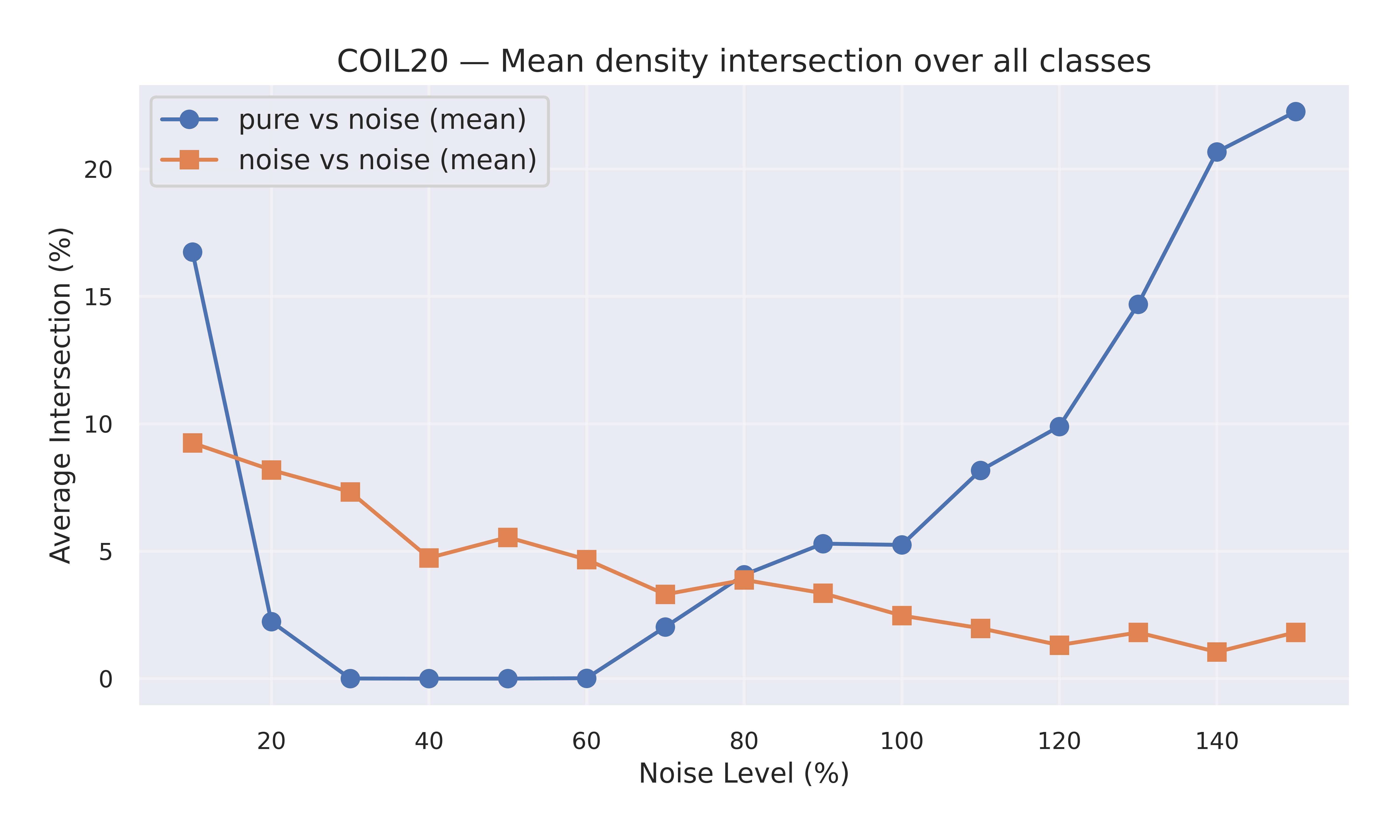}%
}{%
    \fbox{\parbox{0.95\linewidth}{\centering\rev{Place the file \texttt{pictures/COIL20\_noise\_sensitivity.png} here.}}}%
}
\caption{\rev{COIL20 sensitivity analysis: mean intersection between cross-persistence density estimates across all classes vs. noise level. Blue curve: \textit{pure vs. noise} (only the right point cloud is noised). Orange curve: \textit{noise vs. noise} (both point clouds are noised).}}
\label{fig:coil20_noise_sensitivity}
\end{figure}


\rev{A plausible mechanism behind the observed effect is that adding noise perturbs pairwise distances and can effectively “thicken” the sampled manifold. In the cross-filtration, this can increase the frequency of topological interactions between a clean cloud and a noised cloud in a class-dependent way, which amplifies differences between classes in the resulting density. At the same time, the \textit{noise vs.
noise} regime tends to preserve overlap reduction more smoothly, which is consistent with the robustness of the MTD-based statistics to symmetric perturbations. This intuition is aligned with observations in computer vision that controlled noise injection can improve robustness and separability by encouraging reliance on stable structures rather than brittle fine-scale artifacts; see, e.g., adaptive/trainable noise injection and parametric noise injection for robustness \cite{LI2024103855, Parametric_noise}}.

\rev{
From a theoretical perspective, this behavior is consistent with the stability theory of persistence diagrams. It is well established that persistence diagrams are stable with respect to perturbations of the underlying metric or point cloud, with the bottleneck distance bounded by the magnitude of the perturbation. In our setting, moderate Gaussian noise induces controlled perturbations of pairwise distances, leading to bounded and structured shifts of cross-persistence features. For intra-class comparisons (e.g., $MTD(Q_1,Q_1)$), stability implies that moderate perturbations preserve the concentration of topological summaries. In contrast, inter-class comparisons (e.g., $MTD(Q_1,Q_s)$) accumulate both intrinsic geometric discrepancy and perturbation-induced variation, which can increase the separation between the corresponding MTD density estimates in intermediate noise regimes. When the noise level becomes dominant relative to the signal, this structured effect vanishes, explaining the eventual loss of separability observed at high noise levels.
}

\rev{
In particular, the existence of noise regimes where the mean overlap approaches zero indicates that the overlap-based comparison can become highly discriminative without modifying the underlying metric, supporting its use as the primary similarity measure in the main experiments.
}

\section{\break Are the Cross-RipsNet Inputs Redundant?}
\label{app:Ablation_CrossRpsNet}

In this section, we investigate whether all inputs to the Cross-RipsNet model are necessary. Specifically, we perform an ablation study in which we remove the input corresponding to the second (right) point cloud-see Fig.~\ref{fig:Pic_of_archs}-to test whether the model can still capture topological discrepancies using only the first point cloud and the combined cloud.

\rev{Our hypothesis is that the observed effect is not due to redundancy in the model architecture, but rather to redundancy in the input representation itself. Since the combined point cloud implicitly contains both $X$ and $Y$, providing the model with the pair $(X \cup Y,\, Y)$ already makes the information about $X$ recoverable in principle. As a result, during training the network can learn to extract the necessary information about the missing cloud implicitly from the combined input and use it for cross-persistence diagram density prediction.}

To evaluate this, we conducted multiple experiments on a synthetic and 3D shapes datasets using different random seeds to ensure the consistency and robustness of the results. In each run, we assessed model performance using the KL divergence on the cross-persistence density prediction task (see Section~\ref{sec:Cross_RipsNet_experiments} for details).

\rev{As shown in Fig.~\ref{fig:boxplot_ablation}, removing the explicit input corresponding to the second point cloud does not lead to a degradation in performance. This behavior should not be interpreted as an architectural flaw of Cross-RipsNet, but rather as a consequence of redundancy in the provided inputs. Importantly, the cross-persistence diagram $Cross-Barcode(P, Q)$ is not symmetric in its arguments, and the left point cloud plays a distinguished role in the construction of the filtration. Therefore, once the combined cloud is available together with one of the individual clouds, the model has access to all information required to reconstruct the asymmetric topological interactions and to predict the corresponding density.}

\begin{figure}[t]
\centering
\subfloat[Synthetic dataset]{%
    \includegraphics[width=0.49\textwidth]{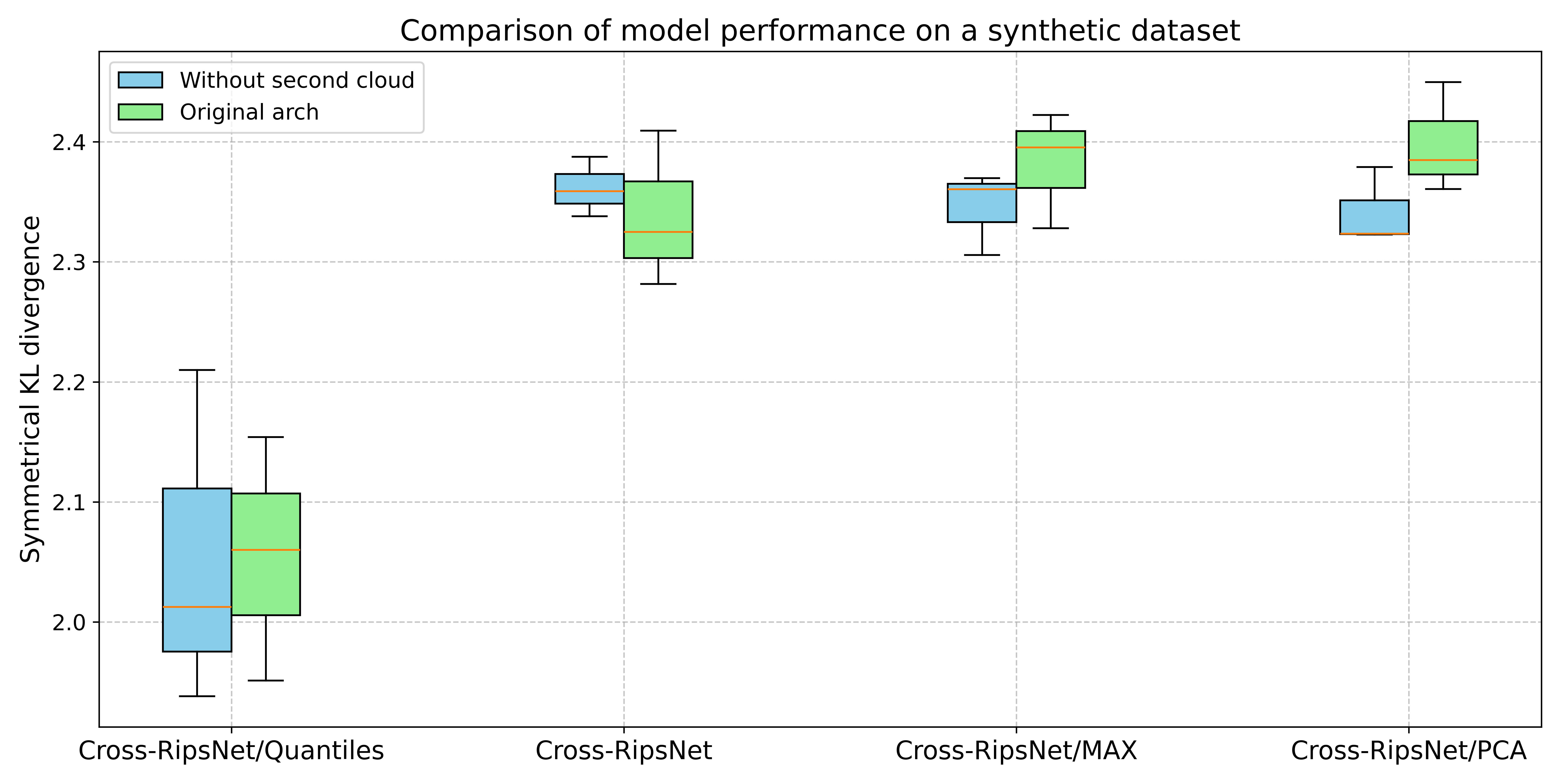}%
    \label{fig:ablation_synthetic}}
\hfil
\subfloat[3D shapes dataset]{%
    \includegraphics[width=0.49\textwidth]{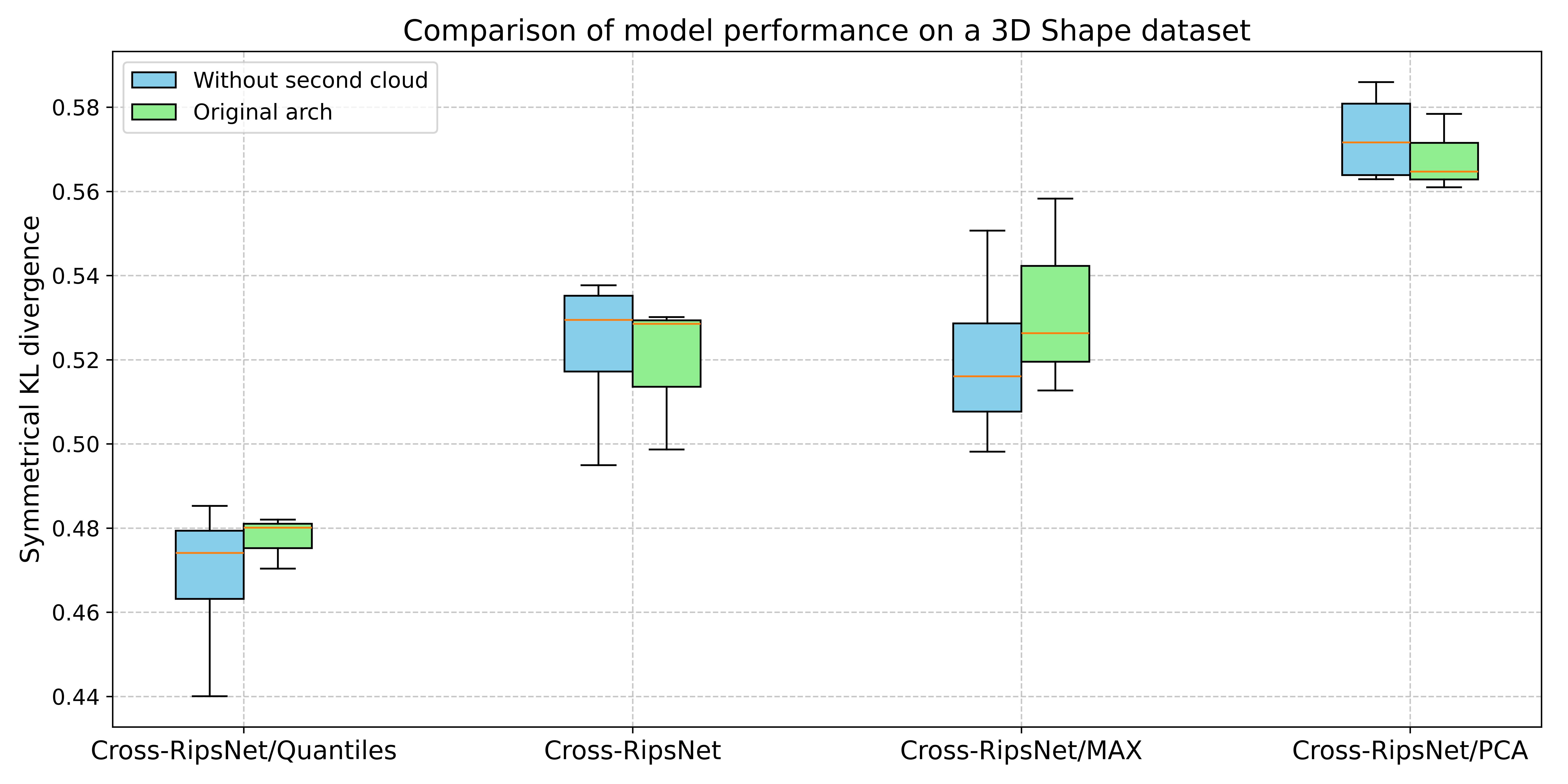}%
    \label{fig:ablation_3d}}
\caption{Boxplot showing the distribution of KL divergence values for cross-persistence density prediction across multiple random seeds on the test set. The plot compares the original Cross-RipsNet model (using all inputs) with a modified version in which the input corresponding to the second point cloud has been removed.}
\label{fig:boxplot_ablation}
\end{figure}

\bibliographystyle{IEEEtran}   
\bibliography{references}      

@inproceedings{3D_shapes,
title={3d shapenets: A deep representation for volumetric shapes},
author={Wu, Zhirong and Song, Shuran and Khosla, Aditya and Yu, Fisher and Zhang, Linguang and Tang, Xiaoou and Xiao, Jianxiong},
booktitle={Proceedings of the IEEE conference on computer vision and pattern recognition},
pages={1912--1920},
year={2015}
}

@inproceedings{Tsybakov2008IntroductionTN,
  title={Introduction to Nonparametric Estimation},
  author={A. Tsybakov},
  booktitle={Springer Series in Statistics},
  year={2008},
  url={https://api.semanticscholar.org/CorpusID:42933599}
}

@inproceedings{Wiki_gpt,
 author = {Krishna, Kalpesh and Song, Yixiao and Karpinska, Marzena and Wieting, John and Iyyer, Mohit},
 booktitle = {Advances in Neural Information Processing Systems},
 pages = {27469--27500},
 title = {Paraphrasing evades detectors of AI-generated text, but retrieval is an effective defense},
 volume = {36},
 year = {2023}
}

@inproceedings{Density-Chazal,
  TITLE = {{The density of expected persistence diagrams and its kernel based estimation}},
  AUTHOR = {Chazal, Fr{\'e}d{\'e}ric and Divol, Vincent},
  URL = {https://hal.science/hal-01716181},
  BOOKTITLE = {{SoCG 2018 - Symposium of Computational Geometry}},
  ADDRESS = {Budapest, Hungary},
  YEAR = {2018},
  MONTH = Jun,
  HAL_ID = {hal-01716181},
  HAL_VERSION = {v3},
}

@inproceedings{tulchinskii2023intrinsicdimensionestimationrobust,
 author = {Tulchinskii, Eduard and Kuznetsov, Kristian and Kushnareva, Laida and Cherniavskii, Daniil and Nikolenko, Sergey and Burnaev, Evgeny and Barannikov, Serguei and Piontkovskaya, Irina},
 booktitle = {Advances in Neural Information Processing Systems},
 pages = {39257--39276},
 title = {Intrinsic Dimension Estimation for Robust Detection of AI-Generated Texts},
 volume = {36},
 year = {2023}
}

@inproceedings{Tulchinskii_2023,
   title={Topological Data Analysis for Speech Processing},
   DOI={10.21437/interspeech.2023-1861},
   booktitle={INTERSPEECH 2023},
   author={Tulchinskii, Eduard and Kuznetsov, Kristian and Kushnareva, Laida and Cherniavskii, Daniil and Barannikov, Serguei and Piontkovskaya, Irina and Nikolenko, Sergey and Burnaev, Evgeny},
   year={2023},
   month=aug, pages={311–315},
   collection={interspeech_2023} }

@article{desurrel2022ripsnetgeneralarchitecturefast,
  TITLE = {{RipsNet: a general architecture for fast and robust estimation of the persistent homology of point clouds}},
  AUTHOR = {Hensel, Felix and Glisse, Marc and Chazal, Fr{\'e}d{\'e}ric and de Surrel, Thibault and Carriere, Mathieu and Lacombe, Th{\'e}o and Kurihara, Hiroaki and Ike, Yuichi},
  URL = {https://inria.hal.science/hal-03867083},
  JOURNAL = {{Proceedings of Machine Learning Research}},
  PUBLISHER = {{PMLR}},
  SERIES = {Proceedings of Topological, Algebraic, and Geometric Learning Workshops 2022},
  VOLUME = {196},
  PAGES = {96-106},
  YEAR = {2022},
  HAL_ID = {hal-03867083},
  HAL_VERSION = {v1},
}

@article{barannikov2021manifold,
  title={Manifold Topology Divergence: a Framework for Comparing Data Manifolds.},
  author={Barannikov, Serguei and Trofimov, Ilya and Sotnikov, Grigorii and Trimbach, Ekaterina and Korotin, Alexander and Filippov, Alexander and Burnaev, Evgeny},
  journal={Advances in neural information processing systems},
  pages={7294--7305},
  year={2021},
}

@article{bubenik2015statistical,
  title={Statistical topological data analysis using persistence landscapes.},
  author={Bubenik, Peter},
  journal={Journal of Machine Learning Research},
  volume={16},
  pages={77--102},
  year={2015}
}

@article{edelsbrunner2002topological,
author = {Zomorodian, Afra J},
journal = {Ph.D. Thesis},
year = {2001},
month = {01},
pages = {},
publisher = {University of Illinois at Urbana-Champaign},
title = {COMPUTING AND COMPREHENDING TOPOLOGY: PERSISTENCE AND HIERARCHICAL {M}ORSE COMPLEXES}
}

@article{adams2016persistenceimages,
  author  = {Henry Adams and Tegan Emerson and Michael Kirby and Rachel Neville and Chris Peterson and Patrick Shipman and Sofya Chepushtanova and Eric Hanson and Francis Motta and Lori Ziegelmeier},
  title   = {Persistence Images: A Stable Vector Representation of Persistent Homology},
  journal = {Journal of Machine Learning Research},
  year    = {2017},
  volume  = {18},
  number  = {8},
  pages   = {1--35},
  url     = {http://jmlr.org/papers/v18/16-337.html}
}

@article{barannikov:morsecomplex,
  TITLE = {{The Framed Morse complex and its invariants}},
  AUTHOR = {Barannikov, Serguei},
  URL = {https://hal.science/hal-01745109},
  JOURNAL = {{Advances in Soviet Mathematics }},
  PUBLISHER = {{American Mathematical Society}},
  SERIES = {Singularities and Bifurcations},
  VOLUME = {21},
  PAGES = {93-116},
  YEAR = {1994},
  MONTH = Apr,
  DOI = {10.1090/advsov/021/03},
  HAL_ID = {hal-01745109},
  HAL_VERSION = {v1},
}

@article{barannikov2022representation,
  TITLE = {{Representation Topology Divergence: A Method for Comparing Neural Network Representations}},
  AUTHOR = {Barannikov, Serguei and Trofimov, Ilya and Balabin, Nikita and Burnaev, Evgeny},
  URL = {https://hal.science/hal-03821864},
  JOURNAL = {{Proceedings of Machine Learning Research}},
  PUBLISHER = {{PMLR}},
  SERIES = {Proceedings of the 39th International Conference on Machine Learning (ICML 2022)},
  VOLUME = {162},
  PAGES = {1607-1626},
  YEAR = {2022},
  MONTH = May,
  HAL_ID = {hal-03821864},
  HAL_VERSION = {v1},
}

@article{zomorodian2005computing,
  title={Computing persistent homology},
  author={Zomorodian, Afra and Carlsson, Gunnar},
  journal={Discrete \& Computational Geometry},
  volume={33},
  number={2},
  pages={249--274},
  year={2005},
  publisher={Springer}
}

@inproceedings{zaheer2018deepsets,
author = {Zaheer, Manzil and Kottur, Satwik and Ravanbhakhsh, Siamak and P\'{o}czos, Barnab\'{a}s and Salakhutdinov, Ruslan and Smola, Alexander J},
title = {Deep Sets},
year = {2017},
booktitle = {Proceedings of the 31st International Conference on Neural Information Processing Systems},
pages = {3394–3404},
}

@inproceedings{kushnareva2021artificial,
  title={Artificial Text Detection via Examining the Topology of Attention Maps},
  author={Kushnareva, Laida and Cherniavskii, Daniil and Mikhailov, Vladislav and Artemova, Ekaterina and Barannikov, Serguei and Bernstein, Alexander and Piontkovskaya, Irina and Piontkovski, Dmitri and Burnaev, Evgeny},
  booktitle={Empirical Methods in Natural Language Processing},
  year={2021}
}

@inproceedings{rieck2020uncoveringtopologytimevaryingfmri,
 author = {Rieck, Bastian and Yates, Tristan and Bock, Christian and Borgwardt, Karsten and Wolf, Guy and Turk-Browne, Nicholas and Krishnaswamy, Smita},
 booktitle = {Advances in Neural Information Processing Systems},
 pages = {6900--6912},
 title = {Uncovering the Topology of Time-Varying fMRI Data using Cubical Persistence},
 volume = {33},
 year = {2020}
}

@article{catch22,
      title={catch22: CAnonical Time-series CHaracteristics}, 
      author={Carl H Lubba and Sarab S Sethi and Philip Knaute and Simon R Schultz and Ben D Fulcher and Nick S Jones},
      journal = {Data Mining and Knowledge Discovery},
      year={2019},
      volume={33},
      number={6},
      pages={1821--1852},
      DOI = {10.1007/s10618-019-00647-x}
}

@InProceedings{FreshPRINCE,
author="Middlehurst, Matthew
and Bagnall, Anthony",
title="The FreshPRINCE: A Simple Transformation Based Pipeline Time Series Classifier",
booktitle="Pattern Recognition and Artificial Intelligence",
year="2022",
publisher="Springer International Publishing",
address="Cham",
pages="150--161",
isbn="978-3-031-09282-4",
DOI={10.1007/978-3-031-09282-4_13},
}

@article{bakeoff-revisited,
   title={Bake off redux: a review and experimental evaluation of recent time series classification algorithms},
   volume={38},
   ISSN={1573-756X},
   url={http://dx.doi.org/10.1007/s10618-024-01022-1},
   DOI={10.1007/s10618-024-01022-1},
   number={4},
   journal={Data Mining and Knowledge Discovery},
   publisher={Springer Science and Business Media LLC},
   author={Middlehurst, Matthew and Schäfer, Patrick and Bagnall, Anthony},
   year={2024},
   month=apr, pages={1958–2031} }

@inproceedings{chazal2013Silhouettes,
author = {Chazal, Fr\'{e}d\'{e}ric and Fasy, Brittany Terese and Lecci, Fabrizio and Rinaldo, Alessandro and Wasserman, Larry},
title = {Stochastic Convergence of Persistence Landscapes and Silhouettes},
year = {2014},
doi = {10.1145/2582112.2582128},
booktitle = {Proceedings of the Thirtieth Annual Symposium on Computational Geometry},
pages = {474–483}
}

@article{LI2024103855,
title = {AdaNI: Adaptive Noise Injection to improve adversarial robustness},
journal = {Computer Vision and Image Understanding},
year = {2024},
issn = {1077-3142},
doi = {https://doi.org/10.1016/j.cviu.2023.103855},
url = {https://www.sciencedirect.com/science/article/pii/S1077314223002357},
author = {Yuezun Li and Cong Zhang and Honggang Qi and Siwei Lyu},
}

@inproceedings{Parametric_noise,
title = "Parametric noise injection: Trainable randomness to improve deep neural network robustness against adversarial attack",
author = "Zhezhi He and Rakin, {Adnan Siraj} and Deliang Fan",
year = "2019",
pages = "588--597",
doi = "10.1109/CVPR.2019.00068",
booktitle = "Proceedings of the IEEE/CVF Conference on Computer Vision and Pattern Recognition",
}

\begin{IEEEbiography}[{\includegraphics[width=1in,height=1.25in,clip,keepaspectratio]{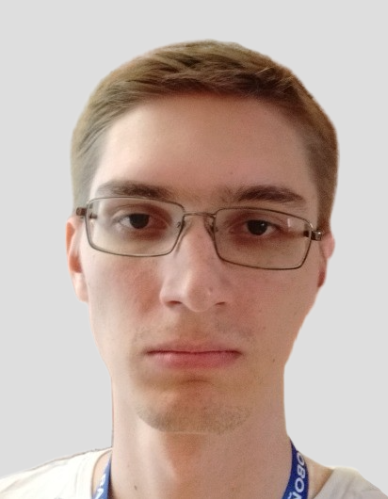}}]{Alexander Mironenko} 
received the M.S. degree in mathematics from State University of Moscow (MSU) in 2023. He is currently a Ph.D. candidate in the Department of Computational and Data Science and Engineering at Skolkovo Institute of Science and Technology, Russia, under the supervision of Prof. Burnaev and Dr. Barannikov. He also works as an engineer at the Optic Algorithm Laboratory, LRC Huawei, Russia.

His research interests include representation learning, topological data analysis, and machine learning.

E-mail: Alexander.Mironenko@skoltech.ru

ORCID iD: 0009-0004-0207-2400
\end{IEEEbiography}

\begin{IEEEbiography}[{\includegraphics[width=1in,height=1.25in,clip,keepaspectratio]{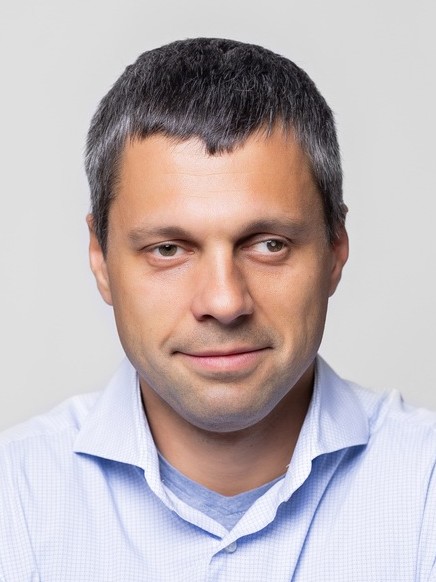}}]{Evgeny Burnaev} 
received the M.S. degree in applied physics and mathematics from the Moscow Institute of Physics and Technology in 2006 and the Ph.D. degree in foundations of computer science from the Institute for Information Transmission Problems RAS in 2008. He is a Doctor of Physico-Mathematical Sciences, Professor of the Russian Academy of Sciences, Vice President for AI Development, Director of the AI Center of Skoltech.

His research interests include Generative Modeling, Manifold Learning, Surrogate Modeling, 3D Deep Learning and Engineering AI.

ORCID iD: 0000-0001-8424-0690
\end{IEEEbiography}

\begin{IEEEbiography}[{\includegraphics[width=1in,height=1.25in,clip,keepaspectratio]{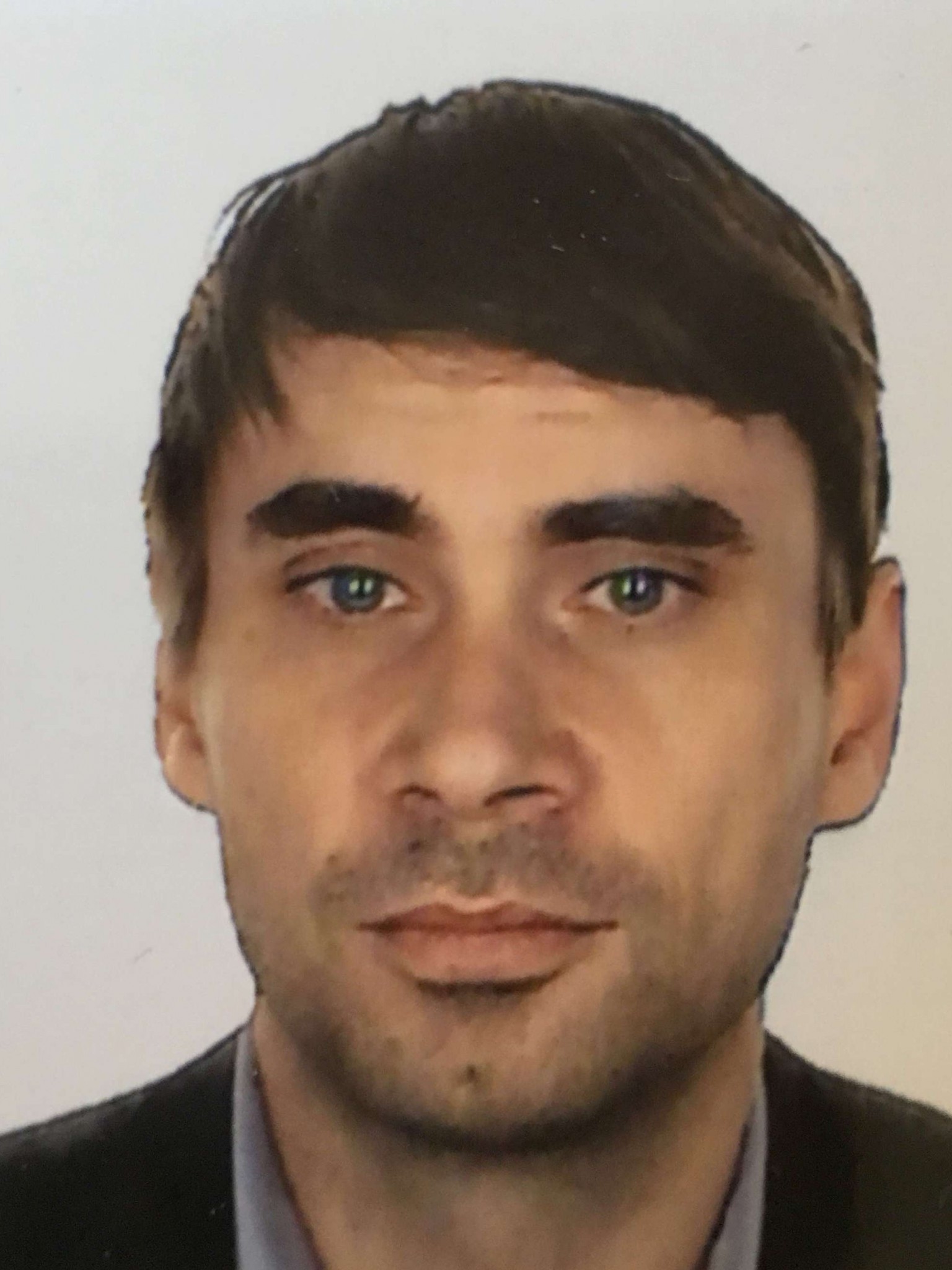}}]{Serguei Barannikov} 
received the M.S. degree in mathematics from State University of Moscow (MSU) and the Ph.D. degree in mathematics from the University of California, Berkeley. He is currently a leading research scientist at Skolkovo Institute of Science and Technology. From 1999 to 2010, he was a researcher at Ecole Normale Supérieure, Paris, and then at Paris Diderot University. His early work introduced canonical forms of filtered complexes, now known as persistence barcodes, fundamental in topological data analysis. More recently, he has applied topological methods to machine learning, including large language models, with publications at NeurIPS, ICML, and ICLR.

His research interests include algebraic topology, algebraic geometry, mathematical physics, and machine learning.

ORCID iD: 0000-0002-9323-0651
\end{IEEEbiography}

\EOD

\end{document}